\documentclass[10pt,twocolumn,letterpaper]{article}
\usepackage{cvpr}
\usepackage{times}
\usepackage{graphicx}
\usepackage{amsmath}
\usepackage{amssymb}
\usepackage{booktabs} 
\usepackage{color}
\usepackage{enumitem}
\usepackage{url}
\usepackage{bm}
\usepackage{amsthm} 
\usepackage{hyperref}
\newtheorem{theorem}{Strategy}

\newtheorem{lemma}{Lemma}
\usepackage{xcolor}

\usepackage{hyperref}

\hypersetup{
    colorlinks,
    linkcolor={red!80!black},
    citecolor={blue!80!black},
    urlcolor={blue!80!black}
}
\usepackage{expl3}
\ExplSyntaxOn
\newcommand\latinabbrev[1]{
  \peek_meaning:NTF . {
    #1\@}%
  { \peek_catcode:NTF a {
      #1.\@ }%
    {#1.\@}}}
\ExplSyntaxOff
\definecolor{red}{rgb}{.9,.2,.2}
\definecolor{green}{rgb}{.2,.9,.2}
\definecolor{blue}{rgb}{.2,.2,.9}

\def\ie{\latinabbrev{i.e}}

\cvprfinalcopy
\title{Weakly supervised learning of indoor geometry by dual warping}

\author{Pulak Purkait~~~~~~~~Ujwal Bonde~~~~~~~~Christopher Zach\\
 Toshiba Research Europe, Cambridge, U.K.\\
{\tt\small \{pulak.cv, ujwal.bonde, christopher.m.zach\}@gmail.com}
}

\begin{document}

\maketitle
 
\begin{abstract}
  A major element of depth perception and 3D understanding is the ability to
  predict the 3D layout of a scene and its contained objects for a novel
  pose. Indoor environments are particularly suitable for novel view
  prediction, since the set of objects in such environments is relatively
  restricted. In this work we address the task of 3D prediction especially for
  indoor scenes by leveraging only weak supervision. In the literature 3D
  scene prediction is usually solved via a 3D voxel grid. 
  However, such methods are limited to estimating rather coarse 3D voxel
  grids, since predicting entire voxel spaces has large computational
  costs. 
  Hence, our method operates in image-space rather than in voxel space, and
  the task of 3D estimation essentially becomes a depth image completion
  problem.
  We propose a novel approach to easily generate training data containing
  depth maps with realistic occlusions, and subsequently train a network for
  completing those occluded
  regions. 
  Using multiple publicly available
  dataset~\cite{song2017semantic,Silberman:ECCV12} we benchmark our method
  against existing approaches and are able to obtain superior performance. We
  further demonstrate the flexibility of our method by presenting results for
  new view synthesis of RGB-D images.
\end{abstract}

\section{Introduction}
\label{sec:intro}
Scene completion has drawn a lot of attention recently from the computer vision~\cite{wu2017marrnet}, robotics~\cite{lai2014unsupervised} as well as the neuroscience community~\cite{gennari1989models}. Most of these works are driven by the assumption that 3D scene completion is important for 3D scene understanding which in turn is useful for tasks such as robot navigation. Towards this end, recent works have addressed 3D scene completion by semantic voxel filling~\cite{song2017semantic,dai2017scannet}. However, these approaches are limited as follows: (i) semantic labeling of 3D voxels will generally produce coarse labelings in order to be computationally feasible, and (ii) labeling of voxels in the object interior might be redundant as one is mostly interested in object surfaces. Therefore, we focus on predicting detailed surfaces rather than semantic voxels. Moreover, methods predicting entire voxel grids rely on large labeled datasets that are expensive to create and label~\cite{song2017semantic}. In contrast, our system does not require any additional labeled data and only relies on calibrated depth images which are easy to acquire using existing RGB-D sensors~\cite{shao2013computer}. An instance of the predicted output is displayed in Fig.~\ref{fig:teaser0}. 
\begin{figure}
\centering 
\centering 
{\includegraphics[width=0.46\textwidth]{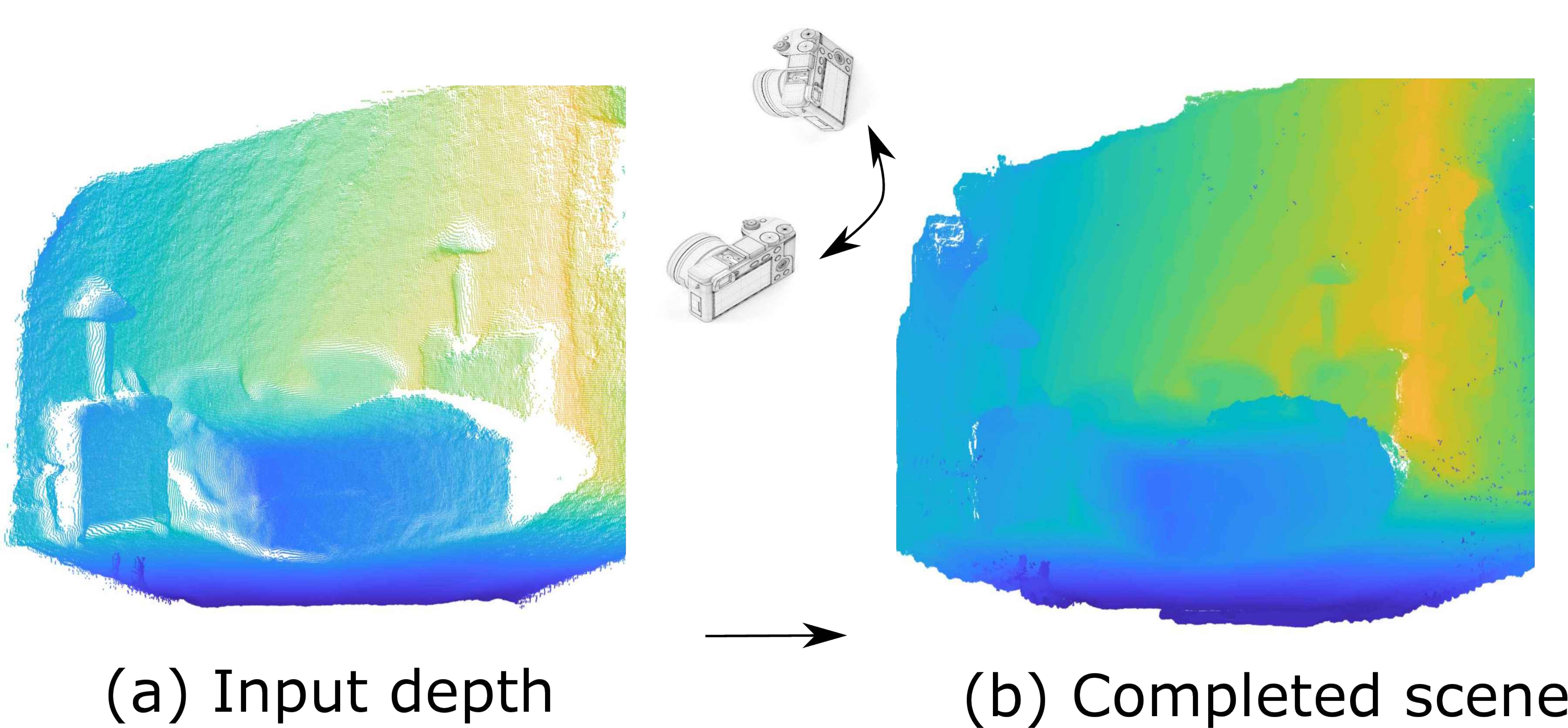}}
\caption{Given an input depth image proposed network can generate multiple depth images which can be further combined for 3D scene completion.}
\label{fig:teaser0}
\end{figure}
%
In summary our contributions are as follows:
\begin{itemize}[leftmargin=1em,itemsep=1pt,parsep=1pt,topsep=1pt]
\item We propose a network architecture to predict depth at arbitrary viewpoints given a single depth image.
\item The network is trained solely using unlabeled depth images without relying on additional supervision signals.
\end{itemize}

\begin{figure*}
\centering 
\includegraphics[width=0.94\textwidth]{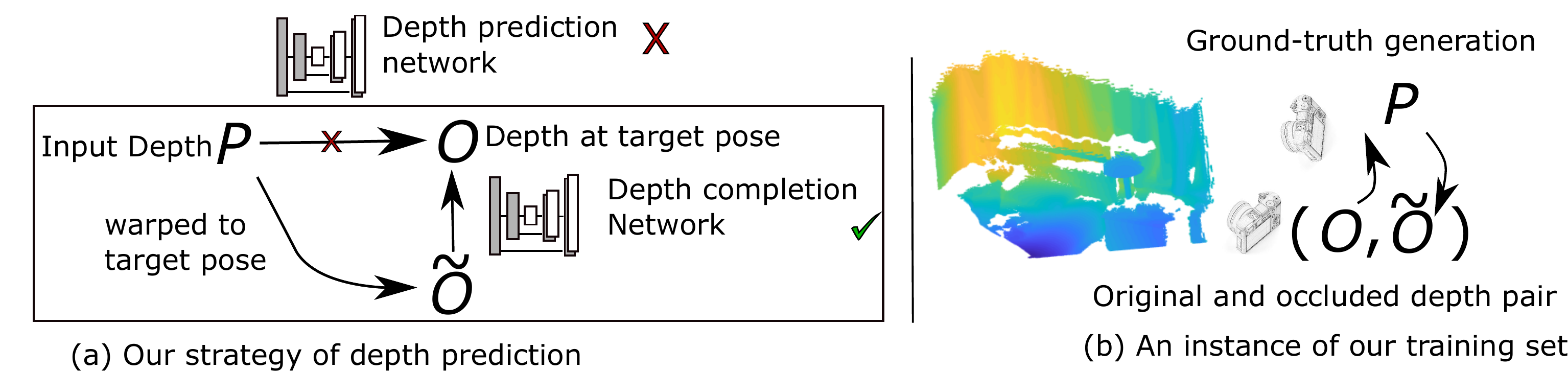} 
\caption{We avoid predicting depth at the target pose, i.e., $P\to$ {\color{red}{depth prediction network}} $\to O$,  by the following depth completion strategy: $P\to$ warp to target pose $\to \tilde{O}\to$ {\color{green}{depth completion network}} $\to O$.}
\label{fig:teaser1}
\end{figure*}

\section{Literature Review}
\label{sec:review}
In computer vision literature the problem of scene completion was addressed by some of the earliest work into human perception understanding~\cite{marr1982computational}, where it is conjectured that human perception relies on its ability to complete scenes. With the availability of cheap sensors~\cite{shao2013computer} and computational resources a renewed interest is seen in this field. In~\cite{song2017semantic,dai2017scancomplete} the authors predict the complete 3D scene from a single depth image using a network to learn shape prior of indoor objects. Given a new scene they predict the voxelised volume with semantic labels for each voxel along with its occupancy probability. Learning these prior requires a large synthetic dataset~\cite{song2017semantic,dai2017scannet} or the need to manually label real world data~\cite{Silberman:ECCV12}, both of which are expensive procedures. In comparison our method solely relies on (synthetically) transforming calibrated depth images to generate training data. Moreover the final prediction in the above-mentioned works is a coarse grid of voxels whereas our system outputs a full-resolution depth image for a desired viewpoint.

Our approach shares similarities with~\cite{zelek2017point} and~\cite{park2017transformation}, which also use displacement fields to solve an image completion/inpainting task. However, these methods are restricted to single or few objects and its unclear how they would generalize to natural scenes. 

We model the scene completion task as an instance of depth image inpainting (e.g.~\cite{iizuka2017globally}). Using low rank approximations is a popular framework for image inpainting~\cite{guillemot2014image}. In~\cite{xue2017depth} the authors extend this approach from RGB to depth images using additional regularization to the gradients. Similar to our work they do not require additional labeled data. However the noise patterns used in image inpainting are often either random and unstructured (e.g.~salt and pepper noise) or structured but artificial (e.g.\ text superimposition). In contrast this work explicitly considers naturally occurring structured missing regions, i.e occlusions generated by warping images based on a given depth map.

The authors of~\cite{pertuz2017region,zhang2018deep} use RGB information to complete a sparse depth image. This is restrictive as we can only complete viewpoints that have corresponding RGB images. On the other hand as we do not rely on RGB images we are able to synthesize arbitrary viewpoints within a reasonably distance from the observed depth map.

\begin{figure*}
\centering 
\begin{tabular}{c@{\hspace{0.05cm}}c@{\hspace{0.05cm}}c@{\hspace{0.05cm}}c}
\includegraphics[width=0.24\textwidth]{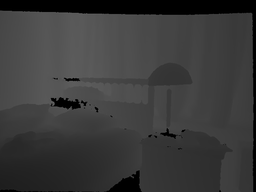} &
\includegraphics[width=0.24\textwidth]{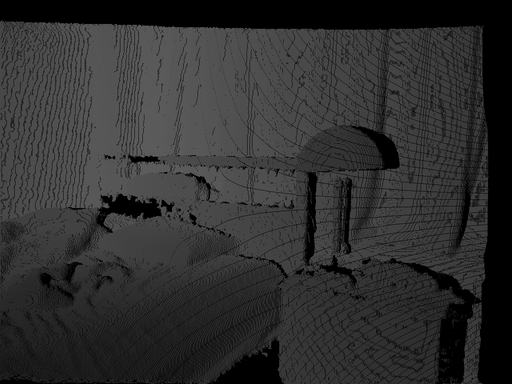} &
\includegraphics[width=0.24\textwidth]{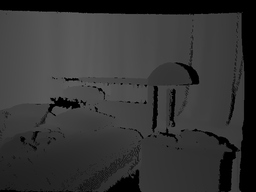} &
\includegraphics[width=0.24\textwidth]{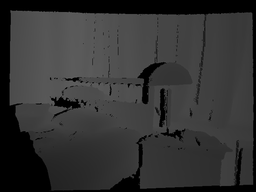} \\
$(a)$  & $(b)$ & $(c)$ & $(d)$
\end{tabular}
\caption{$(a)$ An instance  of a real depth image from NYU Dataset~\cite{Silberman:ECCV12} where pixels with unknown depths (NaN) are marked with zero. $(b)$ Depth after warped forth and back from the original depth to a random location and orientation. $(c)$ Same as $(b)$ but projected 
with upscaling with resolution factor $2$ to reduce the aliasing affect. $(d)$ An additional instance of $(c)$ with opposite view-points. Note that the pairs $\left((c),(a)\right)$ and $\left((d),(a)\right)$ serve as the ground truths for our depth completion network.}
\label{fig:teaser2}
\end{figure*}

\section{Dataset generation via dual warping}
Given a depth image at a particular pose our target is to generate depth views at arbitrary locations. A natural way to accomplish this would be to generate a pair of depth images of the same scene from different view-points and use them as ground truth. However, for this we would require the complete 3D model of a scene from which synthetic pairwise views could be generated, i.e.\ two different depth images of the same scene with the ground truth poses. These 3D models are rarely available or time-consuming to acquire.

In this work we exploit a novel strategy to generate training data solely from given depth images. Let $P$ be a given depth image and $O$ be the depth image at the target pose that we want to estimate.
The estimation of the depth image at a novel view-point $P\to O$ consists of is modelled via two stages: (i) a geometric warping step and (ii) filling of the 
occluded regions. The former is very straight-forward and can be computed efficiently. Thus, we pose the novel view generation 
as a depth completion problem (Fig.~\ref{fig:teaser1}). A convolutional network is trained for this task. 

A strongly supervised approach requires training data consisting of depth map pairs $(O, \tilde O_P)$, where $\tilde O_P$ is the depth map $P$ warped to the pose of $O$, and the task of the DNN is to fill in missing depth values in $\tilde O_P$ to match $O$.
It therefore requires acquisition of multiple (at least two) depth maps for each scene, and a strongly supervised method is consequently not applicable on e.g.\ unordered collections of unrelated depth maps.
Thus, we replace the strongly supervised task by a weakly supervised one, which---as a by-product---turns out to be also less challenging in terms of problem difficulty (see below).
Let $\tilde O$ be the depth map obtained by warping $\tilde P_O$ (i.e.\ $O$ warped to the pose of $P$) back to the pose of $O$, then the training data consists of pairs $(O, \tilde O)$.
Since a given depth map is warped twice we call it ``dual warping'' (see Fig.~\ref{fig:teaser1}).
It only requires independent depth images $\{O\}$ and a method to generate realistic nearby poses (corresponding to the poses of depth maps $\{P\}$, if they were supplied). Thus, we state our first strategy for training data generation below:


\begin{theorem}
The occlusion $O \setminus \tilde{O}$ generated by warping forth and back serves us the ground truth occluded and complete image pair $(\tilde{O},\, O)$. 
\label{th:strategy}
\end{theorem} 
\noindent Note that some parts of the depth map $O$ become occluded during the ``dual warping''. Further, ${O}_y = \tilde{O}_y$ for all pixels $y$ with visible depth $\tilde{O}_y>0$, and e.g.\ $\tilde{O}_y = 0$ for occluded pixels. 
To this end, one can raise a fundamental question: why does one require a complex strategy~\ref{th:strategy} to train a depth completion network.  A straight-forward choice would be following one:
\begin{theorem}
Removing random regions at arbitrary pixel locations in the depth images---the occluded and the original depth image pair $(\tilde{O},\, O)$ 
can serve as the ground truth for depth image completion. 
\label{th:strategy2}
\end{theorem} 
\noindent However, we argue that there is a clear shift in domains between the training data (where random missing regions are presented to the network) and test data (where missing regions are occurring due to occlusions).
We claim (and experimentally verify) that strategy~\ref{th:strategy} brings the training distribution closer to the test distribution, and its properties are further discussed in Sec.~\ref{sec:analysis}. In Sec.~\ref{sec:validation} we also validate that strategy~\ref{th:strategy2} performs inferior to our dual warping strategy~\ref{th:strategy} for dataset generation. 

\subsection{Warping procedure}
Let $(x, y)$ be the original pixel coordinates of the depth image and $K = [f, 0, \bar{x}; 0, f, \bar{y}; 0, 0, 1]$ be the camera matrix where $f$ is the 
focal length and $(\bar{x}, \bar{y})$ are the principle point of the camera. Further, let $s_{xy}$ be the depth at pixel $(x, y)$. The depth 
at the corresponding pixel $(x^{\prime}, y^{\prime})$ at the relative pose $(R, T)$ can be written as 
\begin{equation}
s_{x^{\prime}y^{\prime}}\mathbf{x^\prime} = K(R K^{-1}s_{xy}\mathbf{x} + T) 
\label{eq:perspective}
\end{equation}
where $\mathbf{x}$ and $\mathbf{x^\prime}$ are the homogeneous pixel co-ordinates at $(x, y)$ and $(x^\prime, y^\prime)$ respectively.
We utilize \eqref{eq:perspective} for forward and backward warping.
Hence, our warping procedure essentially corresponds to rendering of 3D point clouds with a z-buffer test enabled for hidden surface removal.
In the following section we postulate and empirically validate that the above strategy will not introduce any additional occlusion (up to aliasing effects due to point instead of mesh rendering). 
Further, the aliasing effect is addressed by upscaling the depth image by a factor of $2$ before warping to the target pose. Note that the camera matrices are 
modified accordingly (\ie $[2f, 0, 2\bar{x}; 0, 2f, 2\bar{y}; 0, 0, 1]$) while warping with the higher resolutions. 
The effectiveness of our warping strategy is demonstrated by the example shown in Fig.~\ref{fig:teaser2}. 

Each depth image is warped to a random pose and then warped back to the original pose. These random poses are generated on the horizontal plane where the translation and orientation increments are uniformly sampled within the range of $[-1m, 1m]$ and $[-15^\circ, 15^\circ]$.
Our synthesized poses thus emulate essentially the lateral motion of a hand-held depth camera.
The axis of the angular shift is chosen as vertical. Each warping generates a pair of original and occluded image. We generate $25$ different original-occluded image pairs for each depth image. Thus the size of our ground truth dataset is $25\times$ the original depth image dataset.

\begin{figure}
  \begin{center}
	\includegraphics[width=0.46\textwidth]{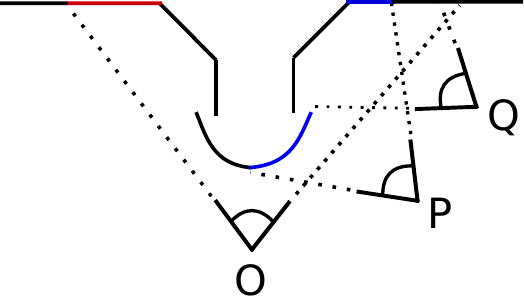} 
  \end{center}
  \caption{An example where $\tilde{O}_P$ (warped forth and back to the pose $P$) and $P^{[O]}$ observe similar objects (marked by blue). 
  In contrast, $\tilde{O}_Q$ (warped forth and back to the pose $Q$) observe more objects than $Q^{[O]}$ (marked by red).} 
  \label{fig:occ}
\end{figure}

\noindent {Note that $\tilde{O}$ and $P^{[O]}$ are depth images observe scene areas visible from both of the poses of $O$ and $P$.  }

\begin{figure*}
\centering 
\includegraphics[width=0.95\textwidth]{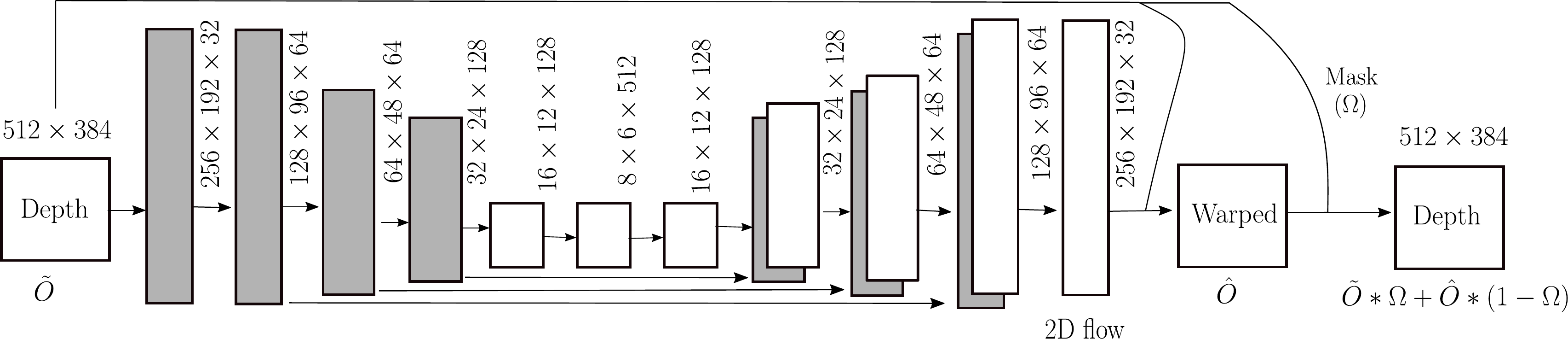} 
\caption{Architecture: The encoder and the decoder is marked by gray and white respectively. The network estimates 2D flow of the source and the 
target pixel locations of similar depth. The depth at the source locations are copied at the target pixel locations with the unknown depth. The depth at the 
known pixel locations are kept unaltered. The network is trained with the training image pairs $(O, \tilde{O})$ and minimize the $\mathcal{L}_{\ell_1}$ loss of the occluded part [see eq.~\eqref{eq:loss}]. No that no direct supervision for the source and target flow is provided and let the network learn the flow directions for which the loss is minimum.}
\label{fig:teaser3}
\end{figure*}

\subsection{Analysis of the generated occlusion patterns}
\label{sec:analysis}
Although in strategy~\ref{th:strategy} we warp twice (to and from a random location), we (essentially) do not introduce any additional occlusions. In fact, the occluded region generated by the strategy is contained in the occluded region generated by warping an actual depth image $P$ at the random pose to the original pose: 
\begin{lemma} 
The occlusion $O \setminus \tilde{O} \subseteq O \setminus P^{[O]}$ where $P^{[O]}$ is the depth image generated by the warping the depth image $P$ to the original pose. 
Further, if the camera at $P$ does not observe any additional objects, then $O \setminus \tilde{O} = O \setminus P^{[O]}$. 
\end{lemma}
\begin{proof}
  Let us denote the warping step of a depth map $O$ into the pose of $P$
  yielding $O^{[P]}$ by $O\xrightarrow[P]{warp} O^{[P]}$, and let $\tilde O :=
  (O^{[P]})^{[O]}$ be the result of dual warping $O\xrightarrow[P]{warp}
  O^{[P]} \xrightarrow[O]{warp}\tilde{O}$. By construction a 3D point $X$
  visible in $O$ (written as $X \in O$) is not visible in $\tilde O$ iff $X
  \notin O^{[P]}$. Compared to $O^{[P]}$ the true depth map $P$ may contain
  additional surfaces occluding $X$ at the pose of $P$, hence $X \notin
  O^{[P]}$ (or $X\notin\tilde{O}$) implies $X \notin P$ (equivalence holds if
  $P$ has no additional occluders not present in $O^{[P]}$ blocking $X$, see
  Fig.~\ref{fig:occ}).\footnote{The vertical line segments are not visible from $O$.} Together with $X\notin P$ implying $X\notin P^{[O]}$ we
  have that $X\notin \tilde{O}$ is a sufficient condition for $X\notin
  P^{[O]}$, or $O\setminus \tilde O \subseteq O\setminus P^{[O]}$.
\end{proof}

\section{Architecture and loss} 
We follow the U-Net architecture very similar to \cite{ronneberger2015u}. 
Its a feed-forward convolutional network consists of an encoder and a symmetric decoder, where a number of skip connections is introduced 
by concatenating the features from the encoder layer to the corresponding decoder layer.   
The network takes an input depth image of size $512\times384$ and 
passes the input through multiple convolutional and de-convolutional layers (with stride $2$) to predict a 2D displacement field of the same size as the input. 
Here the displacements (similar to pixel-shifts~\cite{yan2018shift}) indicate shifts between the source pixel and target pixel location.
Note that target pixel locations are the pixels with missing depth.
The task of the network is to predict the corresponding source pixel locations (from which the known depth value is subsequently copied) instead of directly hallucinating the depth value at the target pixel location.
The details of the architecture can also be found in Fig.~\ref{fig:teaser3}. 
Note that experimentally we observe (validated in the result section) that displacement estimation network performs better than the direct depth prediction network.
In contrast to the single image depth prediction networks~\cite{eigen2014depth} (RGB to depth), in our case the depth for unknown regions is indirectly estimated by copying from known depth map portions.

We utilize a masked $\ell_1$ loss $\mathcal{L}_{\ell_1}$ in this work. 
The mask $\Omega$ is considered as the pixels with the unknown depths $\Omega := O \setminus \tilde{O}$.
We also incorporate a total variation loss $\mathcal{L}_{tv}$ to ensure smooth depth predictions, and we further leverage a content loss~\cite{johnson2016perceptual}, $\mathcal{L}_{c}$, to preserve structure of the depth image as described below:
\begin{equation}
\mathcal{L}_{tv} = \sum_{y\in \Omega}{\left(\|O_y-\hat{O}_y\|_1 + \lambda \|\nabla \hat{O}_y \|_{1} \right) }  
\end{equation}
\begin{equation}
\mathcal{L}_{c} = \gamma \sum_{y\in \Omega} \|\phi_l{({O})}_y - \phi_{l}{(\tilde{O})}_y\|_1 
\label{eq:loss}
\end{equation}
where $\hat{O}_y$ is the predicted depth at the pixel $y$ and $\phi_l$ are the feature descriptors at the layer $l$. Note that the depth is predicted 
only at the unknown pixels and the feature descriptors in the loss \eqref{eq:loss} is only considered for the last two layers. The network is trained to minimize the sum of the above loss $\mathcal{L}_{\ell_1} = \mathcal{L}_{tv} + \mathcal{L}_{c}$. $\lambda$ and $\gamma$ are chosen as $10^{-3}$ and $10^{-5}$ respectively.

\section{Experiments}

The proposed depth completion network (named as {\color{green}{Depth-Flow-Net}}) is evaluated on the widely used SUNCG~\cite{song2017semantic} and NYU Depth v2~\cite{Silberman:ECCV12} datasets. 
The loss $\mathcal{L}_{\ell_1}$ is minimized using ADAM with a mini-batch of size  $10$. 
The weight decay is set to $10^{-5}$. The network is trained for $100$ epochs with an initial learning rate $0.001$ which 
is gradually decreased by a factor of $10$ after every $10$ epochs. 
The network is trained with Tensorflow on a desktop equipped with a NVIDIA Titan X GPU, 
and evaluated on an Intel CPU of $3.10GHz$. 

\noindent {\bf Baseline Methods}
We compare the proposed network against the following baselines: 
\begin{itemize}[leftmargin=1em,itemsep=1pt,parsep=1pt,topsep=1pt]
\item The straight forward network for  
predicting depth directly (named as {\color{green}{Depth-Net}}) instead of predicting depth-flow (Depth-Flow-Net) at the unknown pixels. Depth-Net is employed as baseline in this work. 
\item Low rank completion ({\color{red}{LR}}) \cite{xue2017depth}: The missing depth values are computed by low rank  matrix completion with low gradient regularization.\footnote{code is available at \href{https://github.com/xuehy/depthInpainting}{https://github.com/xuehy/depthInpainting}} 
\item Highly sparse inpainting ({\color{red}{Semantic}})~\cite{pertuz2017region}: Region-based depth recovery for highly sparse depth maps.\footnote{code is available at 
 \href{https://uk.mathworks.com/matlabcentral/fileexchange/64546-depth-map-inpainting}{https://uk.mathworks.com/fileexchange/64546}} This method requires
 semantic labels of different objects present in the depth image. Fortunately, the datasets used in this work contain 
 semantic labels which have been employed during the evaluation of this baseline. Note that none of the other methods including 
 ours do not require semantic labels. 
 \item PDE-based inpainting ({\color{red}{PDE}})~\cite{shen2002mathematical}: Partial differential equation based anisotropic diffusion model for image inpainting is executed as baseline for RGB inpainting. 
 \item Mumford-Shah inpainting ({\color{red}{MS}})~\cite{esedoglu2002digital}: The traditional image inpainting based on Mumford-Shah-Euler model is also evaluated 
 as baseline.\footnote{code is available at \href{https://uk.mathworks.com/matlabcentral/fileexchange/55326-matlab-codes-for-the-image-inpainting-problem}{https://uk.mathworks.com/fileexchange/55326}} 
 
\end{itemize} 

\subsection{Depth image completion}

SUNCG~\cite{song2017semantic} is a large-scale synthetic dataset contains $45,622$ depth images of different scenes with realistic rooms and furniture layouts. The NYU depth dataset~\cite{Silberman:ECCV12} consists of $1,449$ real depth images of indoor environment of commercial 
and residential buildings. The datasets also consists of semantic object labels which are not utilized in this work. 
In each epoch we select a batch of $2,000$ depth image pairs (original and with occlusions) generated by our augmentation technique. More augmented images are generated by random cropping and 
flipping the images in the left/right direction. 

\paragraph{Quantitative Evaluation}
The proposed network is evaluated for the task of generating new depth views. For this task a set of $100$ testing image pairs of the same scene at different view-points and orientations is generated from SUNCG~\cite{song2017semantic} datasets.\footnote{code is available at \href{https://github.com/shurans/sscnet}{https://github.com/shurans/sscnet}}  
One is considered as the depth at the source pose and the other considered as depth at the target pose. 
The depth image at the source pose is first warped at the target pose and then fed into the depth completion network Depth-Flow-Net. In Table~\ref{tab:depthcomplSUNCG}, we display the mean and median depth prediction error evaluated only at the unknown pixels. 

For NYU Depth v2~\cite{Silberman:ECCV12} datasets, no complete 3D model is available.  
Thus, we rely on the ``dual warping'' technique (strategy~\ref{th:strategy}) to generate test data. In Table~\ref{tab:depthcomplNYU} we observe that existing depth inpainting algorithms~\cite{xue2017depth} and~\cite{pertuz2017region}
perform comparably, whereas the proposed depth prediction method improves the median error significantly.  A detailed description of the runtime can also be found in Table~\ref{tab:runtime}. We observe Depth-Flow-Net is fastest among the depth completion benchmark methods. Note that all the methods including ours are evaluated on a CPU. The estimation of 2D displacements, depth completion, and generation of new view are included in the runtime.  

\begin{table}\setlength{\tabcolsep}{6pt}
\centering 
\caption{Depth completion Comparison : SUNCG datasets~\cite{song2017semantic}}
  \begin{tabular}{p{2.3cm}ccc} \toprule  
 & LR \cite{xue2017depth} & Semantic \cite{pertuz2017region} & Ours \\ \midrule 
 Mean error & $0.34$m & $0.33$m & $\mathbf{0.28}$m \\
 Median error & $0.25$m & $0.23$m & $\mathbf{0.05}$m \\
  \bottomrule
  \end{tabular}
\label{tab:depthcomplSUNCG}
\end{table}

\begin{table}\setlength{\tabcolsep}{6pt}
\centering 
\caption{Depth completion Comparison : NYU datasets~\cite{Silberman:ECCV12}}
  \begin{tabular}{p{2.3cm}ccc} \toprule  
 & LR \cite{xue2017depth} & Semantic \cite{pertuz2017region} & Ours \\ \midrule 
 Mean error & $0.42$m &  $\mathbf{0.37}$m & $0.38$m \\
 Median error & $0.26$m & $0.20$m  & $\mathbf{0.06}$m \\
  \bottomrule
  \end{tabular}
\label{tab:depthcomplNYU}
\end{table}

\paragraph{Qualitative Evaluation}
The network is again used to evaluate for the task of generating new depth views. Separate sets of depth images 
are chosen as test sets. Each depth image of the test set is warped w.r.t.\ a randomly sampled target pose (with position and orientation variations from the range of $[-1m, 1m]$ and $[-15^\circ, 15^\circ]$). The proposed depth completion network 
is then employed for depth completion at the occluded regions. 
The novel view generation results are plotted in Fig.~\ref{fig:results2}. 
Although, direct depth estimation network Depth-Net produces reasonable solutions in some cases, we observe that proposed dispacement-field based Depth-Flow-Net produces more consistent solutions. The estimated displacements are displayed by {\color{green}{green}} lines segments while target pixels are marked with {\color{blue}{blue}} dots.

In order to enhance the quality of completed depth maps, we utilize an ensemble-inspired framework:
after warping the current depth maps to multiple nearby poses, the induced occlusions are completed using Depth-Flow-Net.
The resulting depth maps are warped back to the original pose and are subsequently merged using a pixel-wise median filter (which we use as a efficient surrogate for a more refined approach for depth maps fusion such as~\cite{merrell2007realtime}).
%
Examples for such 3D scene completion can be found in Fig.~\ref{fig:results2}(c) and Fig.~\ref{fig:results2}(e). 
Note that signed distance functions (i.e.\ volumetric fusion) could be applied to obtain smoother surfaces.
However, we leave this for future work. More results can be found in the supplementary material.

\subsection{Validation of dual warping for dataset creation}
\label{sec:validation}
To validate our synthetic dataset generation method (strategy~\ref{th:strategy}), we conduct an experiment with a dataset generated by strategy~\ref{th:strategy2}. For each depth image the missing regions are generated by removing random regions of pixels. 
Up to $20\%$ of all pixels are removed. The size of each removed region is chosen uniformly within the range $[1, 50]$ along both the directions. We train Depth-Net for both the datasets generated by strategy~\ref{th:strategy} 
and~\ref{th:strategy2}. Note that in the current experiment 
Depth-Net is chosen over Depth-Flow-Net to demonstrate that strategy~\ref{th:strategy} does not just favor a displacement-based network, but
enhances the problem itself. The results are displayed in Fig.~\ref{fig:alternativestrategy}. We observe more accurate 
depth prediction with strategy~\ref{th:strategy}. Thus the current experiment validates our argument of minimal domain shift of novel view generation with strategy ~\ref{th:strategy}. 

\begin{figure}
\centering \scriptsize  
\begin{tabular}{c@{\hspace{0.05cm}}c@{\hspace{0.05cm}}c}
\includegraphics[width=0.15\textwidth]{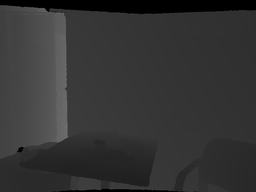} & 
\includegraphics[width=0.15\textwidth]{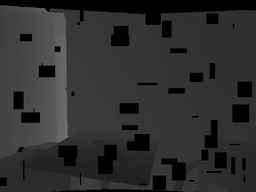} & 
\includegraphics[width=0.15\textwidth]{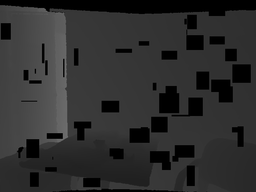} \\  
(a) A sample depth & (b.1) synthetic train data & (b.2) synthetic train data \\ 
\includegraphics[width=0.15\textwidth]{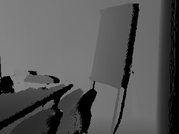} & 
\includegraphics[width=0.15\textwidth]{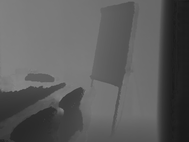} &
\includegraphics[width=0.15\textwidth]{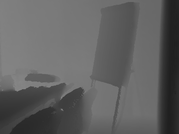} \\ 
(c) Input & (d) trained with Strategy~\ref{th:strategy} & (e) trained with Strategy~\ref{th:strategy2}
\end{tabular}
\caption{Experiment with alternative strategy: Removing random blocks. We observe more ghosting artifacts if the same network (Depth-Net) is trained with the alternative strategy compared to the ``dual warping'' strategy.}
\label{fig:alternativestrategy}  
\end{figure}

\subsection{Limitation / Failure cases}
Despite the generally good performance of Depth-Flow-Net we have encountered failure cases, usually caused by the following:
\begin{itemize}[leftmargin=1em,itemsep=1pt,parsep=1pt,topsep=1pt]
\item In the presence of large occlusions on foreground objects (e.g.~Fig.~\ref{fig:failurecases}(a)) the background depth may incorrecly spill over into the foreground object (Fig.~\ref{fig:failurecases}(b,c)).
\item Very large occlusions (or otherwise regions with missing depth, such as in Fig.~\ref{fig:failurecases}(d)) can lead to displacement vectors that point themselves to missing data (Fig.~\ref{fig:failurecases}(e,f)). In our current approach there is no guarantee that the displacement field always refers to valid depth.
\end{itemize}
Despite the above limitations, proposed Depth-Flow-Net produces satisfactory results in a wide variety of depth images. A number 
of examples are included in the supplementary material. 


\subsection{Novel RGBD image synthesis}
We also exploit our image augmentation strategy for new view RGBD image synthesis given a single RGBD image. We utilize a similar augmentation (strategy~\ref{th:strategy}) to generate 
the ground truth for RGBD image completion. A network similar to Depth-Flow-Net is trained on the augmented RGB-D datasets 
of (original-occluded) pairs. In contrast to depth estimation it takes 4D channels as input and estimate 2D displacements
from source to target regions. Once the network is trained we warp the original view to the target views and complete the missing pixels using the predicted displacement field. 

We conduct similar procedure as before for quantitative and qualitative evaluation. A quantitative comparison can be found in Table~\ref{tab:rgbcompl} and a qualitative comparison is displayed in Fig.~\ref{fig:resultsrgb}.  We observe an improvement 
of PSNR compared to the traditional in-painting algorithms. More results can be found in the supplementary document. 
  
  \begin{figure}
\centering \scriptsize  
\begin{tabular}{c@{\hspace{0.05cm}}c@{\hspace{0.05cm}}c}
\includegraphics[width=0.15\textwidth]{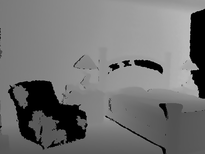} & 
\includegraphics[width=0.15\textwidth]{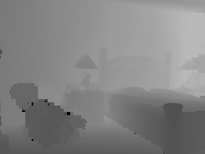} & 
\includegraphics[width=0.15\textwidth]{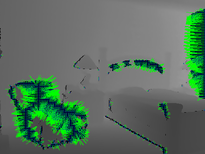} \\  
(a) A sample depth & (b) Completed depth & (c) Estimated flow \\ 
\includegraphics[width=0.15\textwidth]{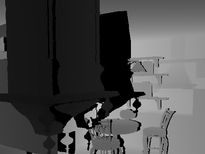} & 
\includegraphics[width=0.15\textwidth]{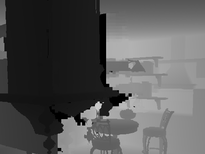} &
\includegraphics[width=0.15\textwidth]{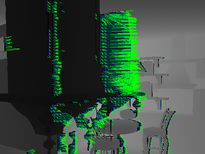} \\ 
(d) A sample depth & (e) Completed depth & (f) Estimated flow \\ 
\end{tabular}
\caption{Failure cases: Two different instances of NYU (a) and SUNCG (d) datasets where proposed Depth-Flow-Net fails due to the 
limitations in the current approach. See text for more details. }
\label{fig:failurecases}  
\end{figure}

\begin{table}\setlength{\tabcolsep}{2pt}
\centering 
\caption{Runtime Comparison: evaluated on CPU}
  \begin{tabular}{p{1.2cm}ccccc} \toprule  
& MS~\cite{esedoglu2002digital} & PDE~\cite{shen2002mathematical} & LR \cite{xue2017depth} & Semantic \cite{pertuz2017region} & Ours \\ \midrule 
 Runtime & $114.8$s & $4.07$s & $73$s & $5.81$s & $\mathbf{0.7}$s \\
  \bottomrule
  \end{tabular}
\label{tab:runtime}
\end{table}

\begin{table}\setlength{\tabcolsep}{10pt} 
\centering 
\caption{Quantitative comparison of RGB image completion}
  \begin{tabular}{p{1.6cm}ccc} \toprule  
 & PDE~\cite{shen2002mathematical} & MS~\cite{esedoglu2002digital} & Ours \\ \midrule 
 PSNR & $24.9$dB & $24.7$dB & $\mathbf{26.08}$dB \\
  \bottomrule
  \end{tabular}
\label{tab:rgbcompl}
\end{table}

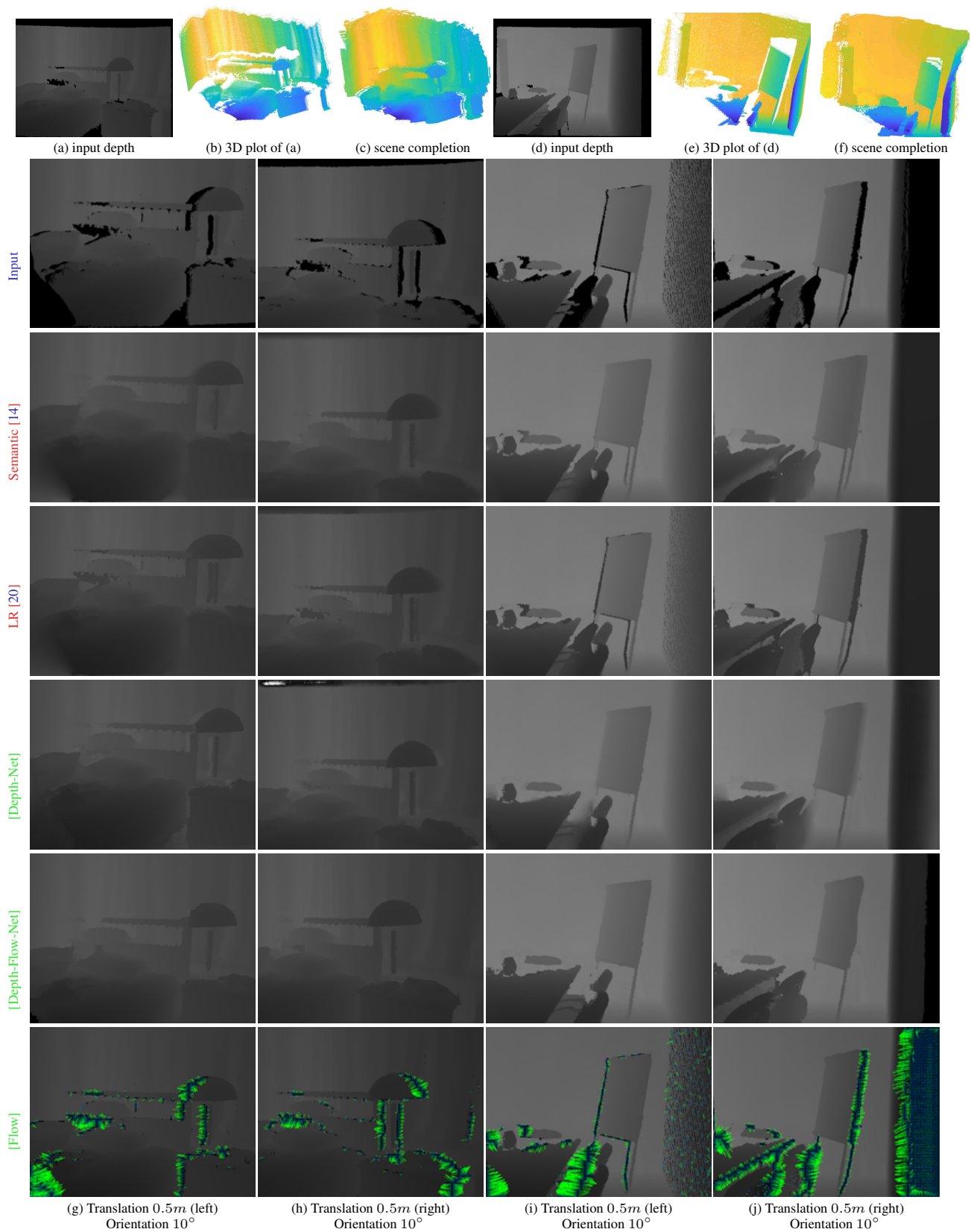
\begin{figure*}
\centering \scriptsize   
\begin{tabular}{l@{\hspace{0.35cm}}c@{\hspace{0.05cm}}c@{\hspace{0.05cm}}c@{\hspace{0.05cm}}c@{\hspace{0.05cm}}c@{\hspace{0.05cm}}c}
& \includegraphics[width=0.16\textwidth]{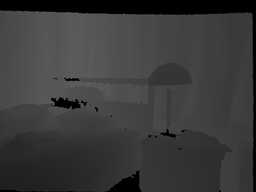}  
& \includegraphics[width=0.16\textwidth]{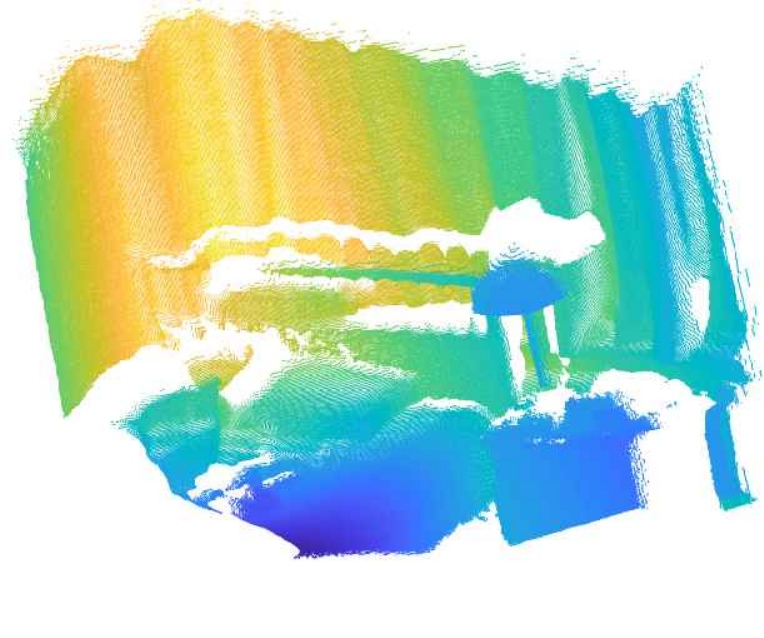}  
& \includegraphics[width=0.16\textwidth]{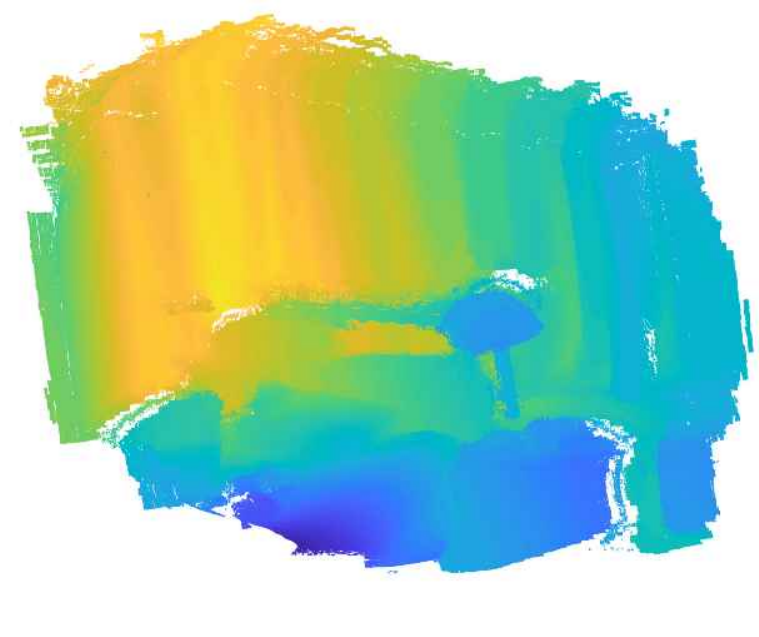}  
 & \includegraphics[width=0.16\textwidth]{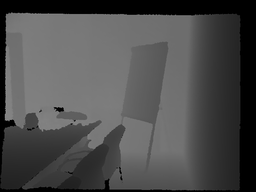} 
 & \includegraphics[width=0.16\textwidth]{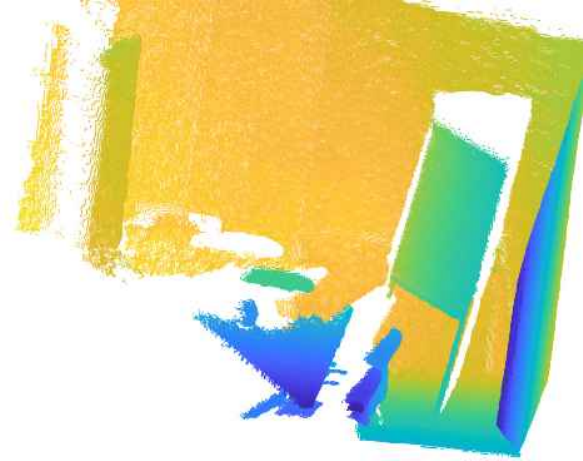}  
& \includegraphics[width=0.16\textwidth]{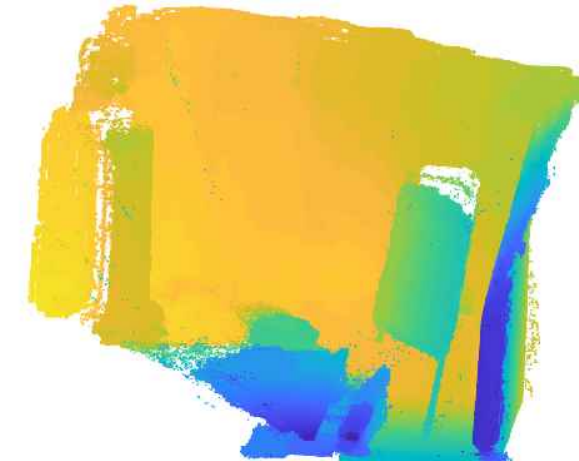}  \\ 
& (a) input depth & (b) 3D plot of (a) & (c) scene completion & (d) input depth & (e) 3D plot of (d) & (f) scene completion
\end{tabular}

\begin{tabular}{l@{\hspace{0.35cm}}c@{\hspace{0.05cm}}c@{\hspace{0.05cm}}c@{\hspace{0.05cm}}c}

\begin{picture}(1,25)
  \put(0,15){\rotatebox{90}{~~~~~\color{blue}{Input}}}
\end{picture} & 
\includegraphics[width=0.23\textwidth]{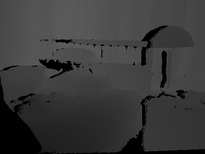} &
\includegraphics[width=0.23\textwidth]{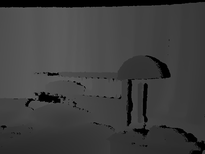} & 
\includegraphics[width=0.23\textwidth]{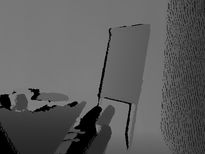} &
\includegraphics[width=0.23\textwidth]{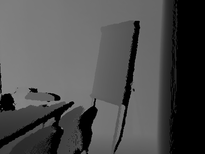} \\ 

\begin{picture}(1,25)
  \put(0,2){\rotatebox{90}{~~~~~\color{red}{Semantic~\cite{pertuz2017region}}}}
\end{picture} & 
\includegraphics[width=0.23\textwidth]{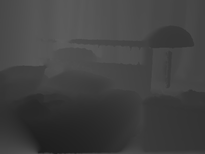} &
\includegraphics[width=0.23\textwidth]{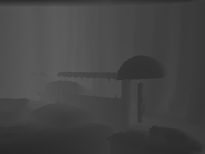} & 
\includegraphics[width=0.23\textwidth]{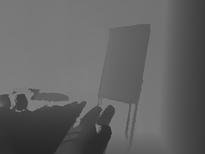} &
\includegraphics[width=0.23\textwidth]{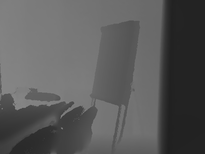} \\ 

\begin{picture}(1,25)
  \put(0,15){\rotatebox{90}{~~~~~\color{red}{LR~\cite{xue2017depth}}}}
\end{picture} & 
\includegraphics[width=0.23\textwidth]{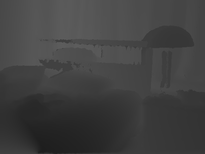} &
\includegraphics[width=0.23\textwidth]{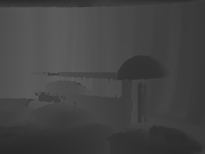} &
\includegraphics[width=0.23\textwidth]{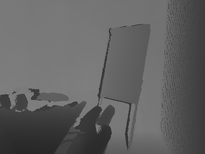} &
\includegraphics[width=0.23\textwidth]{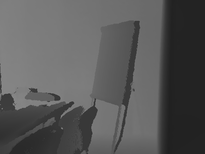}  \\

\begin{picture}(1,25)
  \put(0,6){\rotatebox{90}{~~~~~\color{green}{[Depth-Net]}}}
\end{picture} & 
\includegraphics[width=0.23\textwidth]{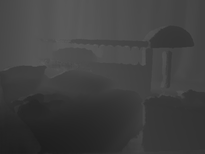} &
\includegraphics[width=0.23\textwidth]{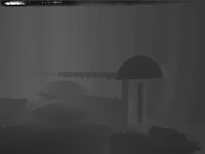} & 
\includegraphics[width=0.23\textwidth]{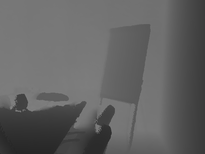} &
\includegraphics[width=0.23\textwidth]{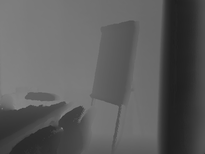} \\

\begin{picture}(1,25)
  \put(0,-2){\rotatebox{90}{~~~~~\color{green}{[Depth-Flow-Net]}}}
\end{picture} & 
\includegraphics[width=0.23\textwidth]{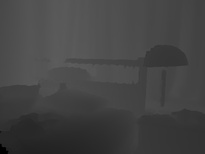} &
\includegraphics[width=0.23\textwidth]{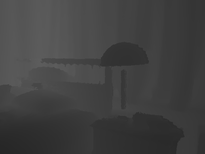} & 
\includegraphics[width=0.23\textwidth]{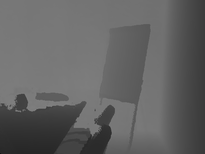} &
\includegraphics[width=0.23\textwidth]{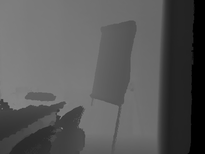} \\ 

\begin{picture}(1,25)
  \put(0,15){\rotatebox{90}{~~~~~\color{green}{[Flow]}}}
\end{picture} & 
\includegraphics[width=0.23\textwidth]{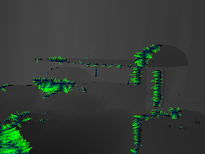} &
\includegraphics[width=0.23\textwidth]{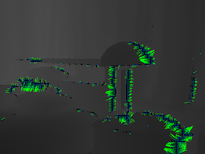} & 
\includegraphics[width=0.23\textwidth]{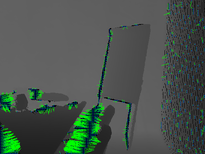} &
\includegraphics[width=0.23\textwidth]{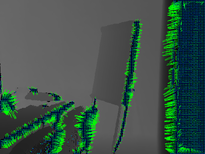} \\ 
& (g) Translation $0.5m$ (left) & (h) Translation $0.5m$ (right) & (i) Translation $0.5m$ (left) & (j) Translation $0.5m$ (right) \\ 
& ~~~~~Orientation $10^\circ$ & ~~~~~Orientation $10^\circ$ & ~~~~~Orientation $10^\circ$  &  ~~~~~Orientation $10^\circ$ 
\end{tabular}
\caption{Qualitative results of depth completion methods at different viewing location and orientation on NYU Depth v2~\cite{Silberman:ECCV12}  dataset. Images are 
first warped to the target pose and then use depth completion methods to predict depth at the occluded regions. 
{\color{green}{Depth-Flow-Net}} produces less artifacts and can even hallucinate handles of the chairs. A complete video 
is shown in the supplementary material.}
\label{fig:results2}
\end{figure*}

\section{Evaluation on SUNCG~\cite{song2017semantic} dataset}

To evaluate the proposed method on  SUNCG~\cite{song2017semantic} datasets, we utilize similar evaluation dataset generation technique as NYU~\cite{Silberman:ECCV12} datasets. We warp the current depth maps to multiple nearby poses and the induced occlusions are completed by proposed Depth-Flow-Net. The results are displayed in Fig.~\ref{fig:results10} and Fig.~\ref{fig:results11} along with the estimated depth flow. Note that the depth-flow is estimated at every pixels but in the figure we only show the flow at the unknown pixels with a regular $4$ pixel interval.  

\section{Conclusion}

In this work we develop a technique to complete 3D scenes indirectly by
filling occluded regions in warped depth (and optionally RGB) images.  Hence,
we are able to avoid a costly volumetric representation and consequently work
in higher-resolution image-space. Our main contribution is the generation of
training data via dual warping, which adds realistic occlusion patterns to
given depth images. Therefore large amounts of training data are easy to
acquire. We also perform a thorough evaluation to demonstrate the
effectiveness of our weakly supervised approach and to show the efficiency of
a proposed depth completion network. Further, the flexibility of the proposed
method is emphasized by an evaluation on RGB-D data. Currently, the proposed
method is limited to generating depth images of relatively nearby poses, which
is a restriction addressed in future research.


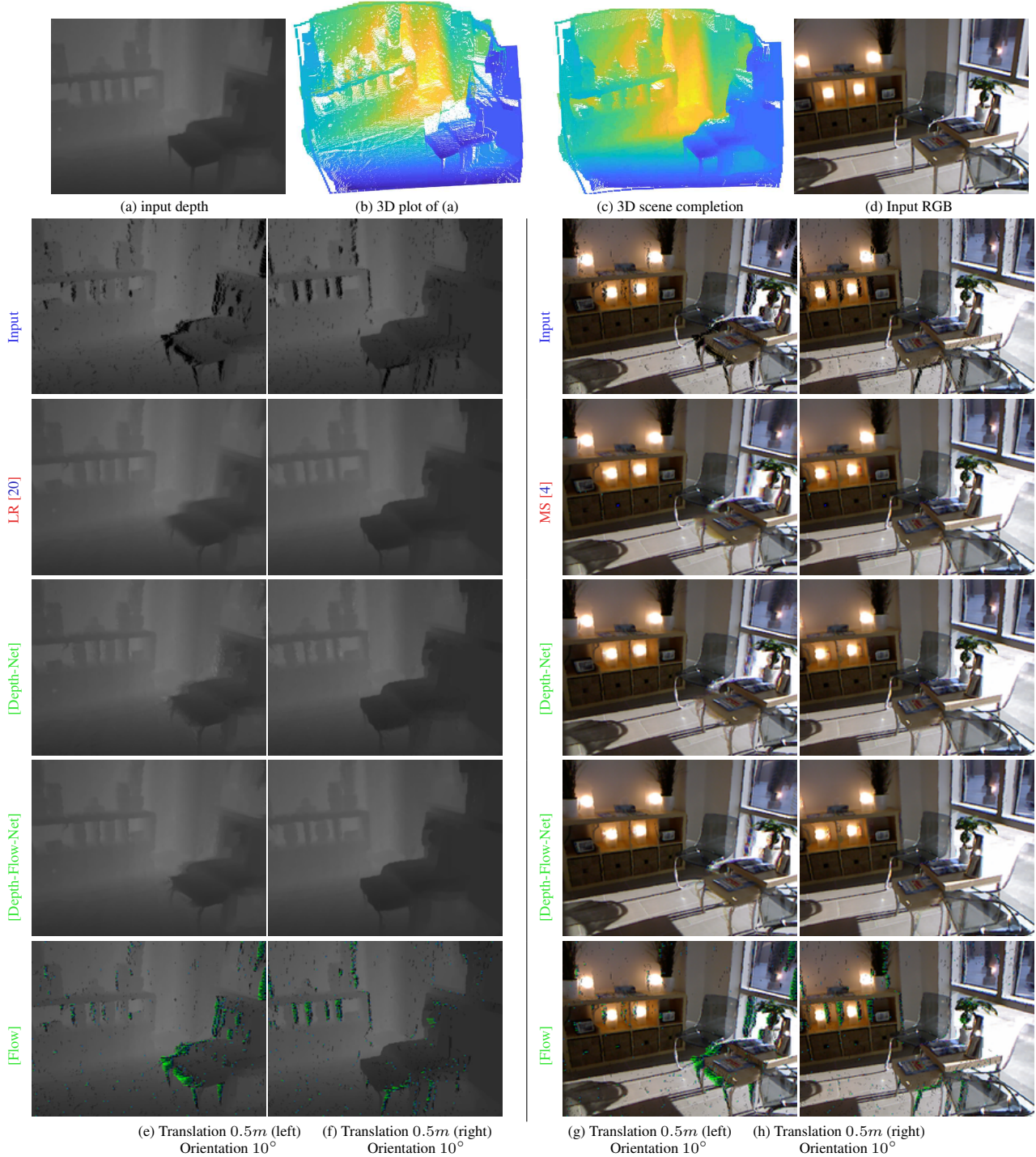
\begin{figure*}
\centering \scriptsize   
\begin{tabular}{@{\hspace{0.7cm}}c@{\hspace{0.05cm}}c@{\hspace{0.45cm}}c@{\hspace{0.05cm}}c}
~ {\includegraphics[width=0.22\textwidth]{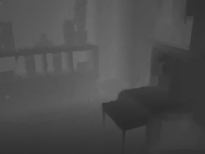}}
   & ~{\includegraphics[width=0.22\textwidth]{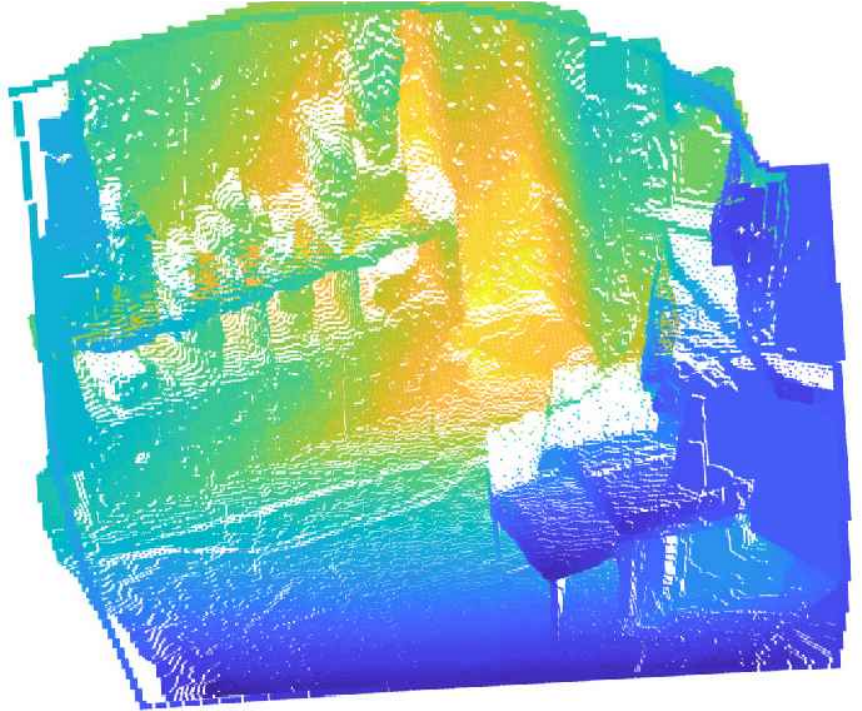}}
   & {\includegraphics[width=0.22\textwidth]{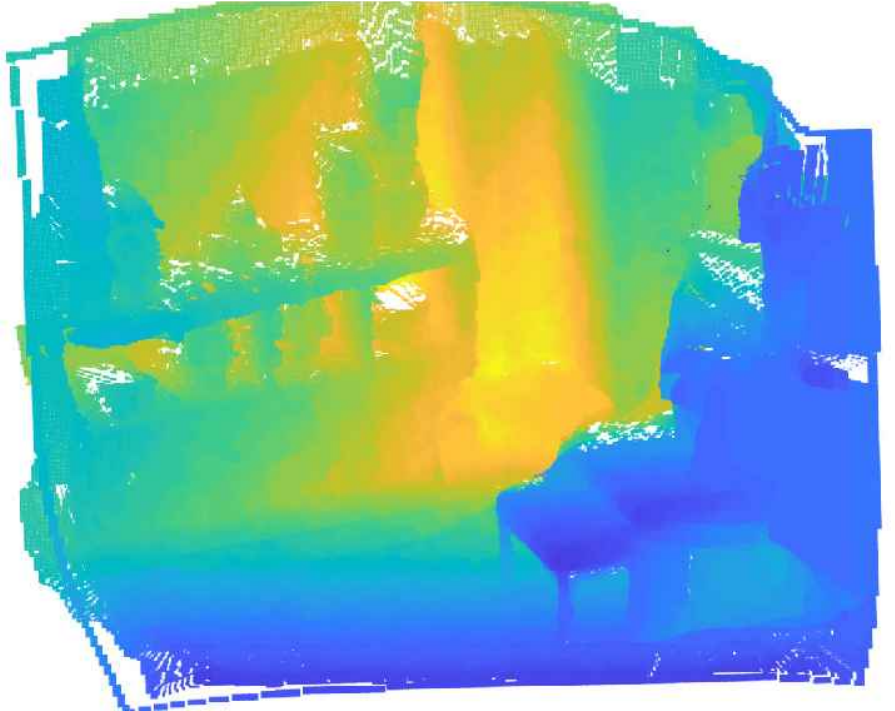}}
   & ~{\includegraphics[width=0.22\textwidth]{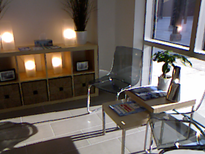}} \\ 
 (a) input depth & (b) 3D plot of (a) & (c) 3D scene completion & (d) Input RGB 
\end{tabular} 
\begin{tabular}{l@{\hspace{0.35cm}}c@{\hspace{0.05cm}}c@{\hspace{0.05cm}}@{\hspace{0.35cm}}|l@{\hspace{0.35cm}}c@{\hspace{0.05cm}}c}

\begin{picture}(1,25)
  \put(0,15){\rotatebox{90}{~~~~~\color{blue}{Input}}}
\end{picture} & 
\includegraphics[width=0.22\textwidth]{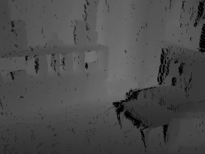} &
\includegraphics[width=0.22\textwidth]{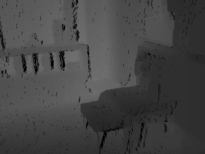} &
\begin{picture}(1,25)
  \put(0,15){\rotatebox{90}{~~~~~\color{blue}{Input}}}
\end{picture} & 
\includegraphics[width=0.22\textwidth]{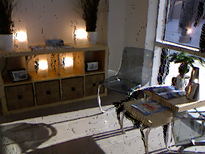} &
\includegraphics[width=0.22\textwidth]{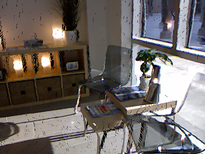} \\ 

\begin{picture}(1,25)
  \put(0,15){\rotatebox{90}{~~~~~\color{red}{LR~\cite{xue2017depth}}}}
\end{picture} & 
\includegraphics[width=0.22\textwidth]{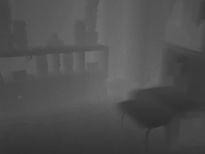} &
\includegraphics[width=0.22\textwidth]{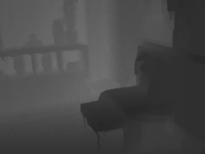} &
\begin{picture}(1,25)
  \put(0,15){\rotatebox{90}{~~~~~\color{red}{MS~\cite{esedoglu2002digital}}}}
\end{picture} & 
\includegraphics[width=0.22\textwidth]{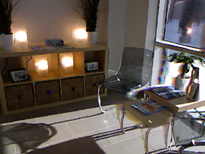} &
\includegraphics[width=0.22\textwidth]{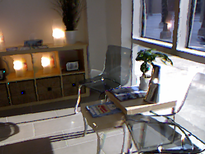} \\ 

\begin{picture}(1,25)
  \put(0,10){\rotatebox{90}{~~~~~\color{green}{[Depth-Net]}}}
\end{picture} & 
\includegraphics[width=0.22\textwidth]{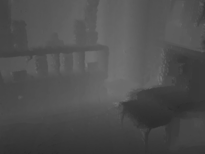} &
\includegraphics[width=0.22\textwidth]{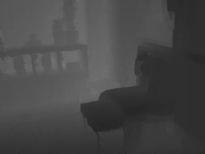} &
\begin{picture}(1,25)
  \put(0,10){\rotatebox{90}{~~~~~\color{green}{[Depth-Net]}}}
\end{picture} & 
\includegraphics[width=0.22\textwidth]{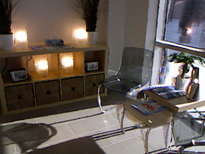} &
\includegraphics[width=0.22\textwidth]{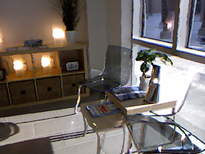} \\ 

\begin{picture}(1,25)
  \put(0,-2){\rotatebox{90}{~~~~~\color{green}{[Depth-Flow-Net]}}}
\end{picture} & 
\includegraphics[width=0.22\textwidth]{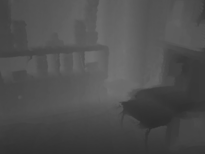} &
\includegraphics[width=0.22\textwidth]{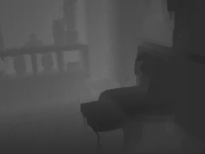} &
\begin{picture}(1,25)
  \put(0,-2){\rotatebox{90}{~~~~~\color{green}{[Depth-Flow-Net]}}}
\end{picture} & 
\includegraphics[width=0.22\textwidth]{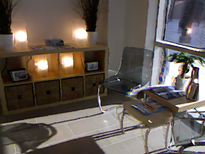} &
\includegraphics[width=0.22\textwidth]{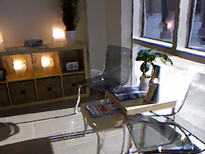} \\

\begin{picture}(1,25)
  \put(0,15){\rotatebox{90}{~~~~~\color{green}{[Flow]}}}
\end{picture} & 
\includegraphics[width=0.22\textwidth]{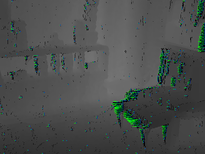} &
\includegraphics[width=0.22\textwidth]{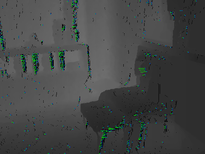} &

\begin{picture}(1,25)
  \put(0,15){\rotatebox{90}{~~~~~\color{green}{[Flow]}}}
\end{picture} & 
\includegraphics[width=0.22\textwidth]{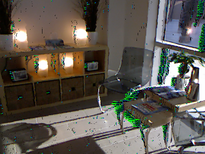} &
\includegraphics[width=0.22\textwidth]{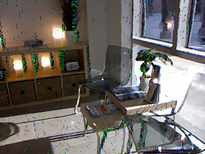} \\

\end{tabular} 
\begin{tabular}{l@{\hspace{0.35cm}}c@{\hspace{0.35cm}}c@{\hspace{0.45cm}}@{\hspace{0.45cm}}l@{\hspace{0.35cm}}c@{\hspace{0.35cm}}c}

&  (e) Translation $0.5m$ (left) & (f) Translation $0.5m$ (right)  & & (g) Translation $0.5m$ (left) & (h) Translation $0.5m$ (right) \\ 
& ~~~~~Orientation $10^\circ$ & ~~~~~Orientation $10^\circ$ & & ~~~~~Orientation $10^\circ$  &  ~~~~~Orientation $10^\circ$\\[1em] 

\end{tabular}
\caption{(a) Qualitative results of depth completion methods at different viewing location and orientation. 
The same network used for RGB image completion to obtain a new RGB image at the target pose. A complete video 
is shown in the supplementary material.
\vspace{2em}
}
\label{fig:resultsrgb}
\end{figure*}

\begin{figure*}
\centering \scriptsize  
\begin{tabular}{c@{\hspace{0.05cm}}c@{\hspace{0.05cm}}c@{\hspace{0.05cm}}c}

\includegraphics[width=0.22\textwidth]{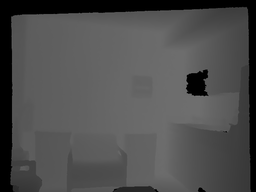} & 
\includegraphics[width=0.22\textwidth]{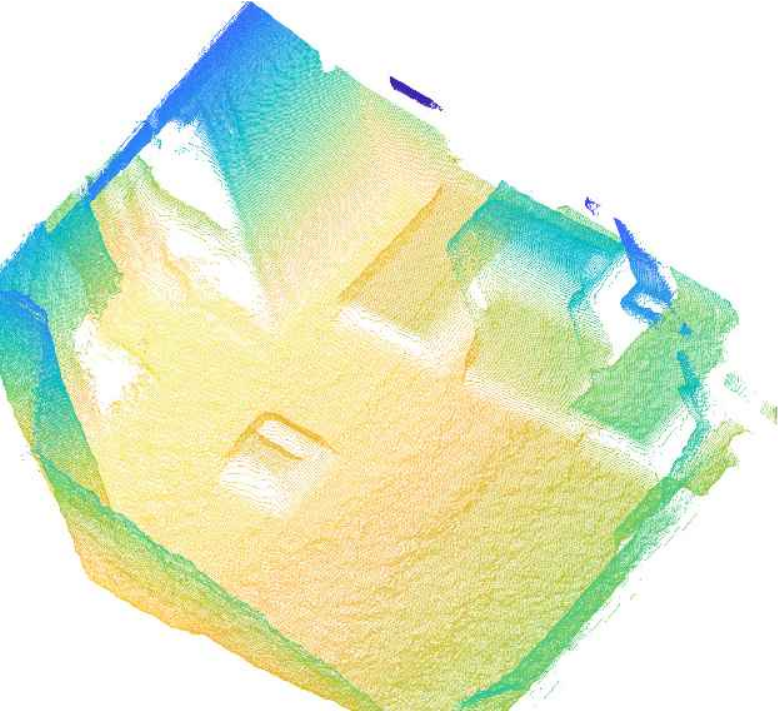} & 
\includegraphics[width=0.22\textwidth]{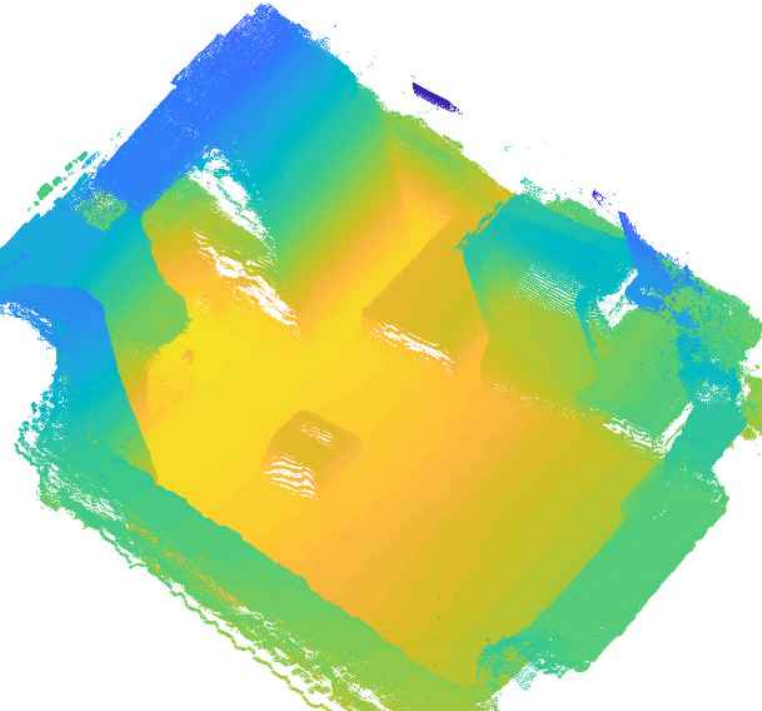} & 
\includegraphics[width=0.22\textwidth]{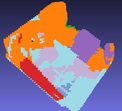} \\

\includegraphics[width=0.22\textwidth]{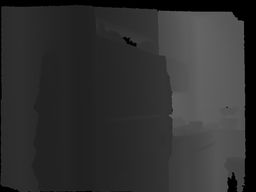} & 
\includegraphics[width=0.22\textwidth]{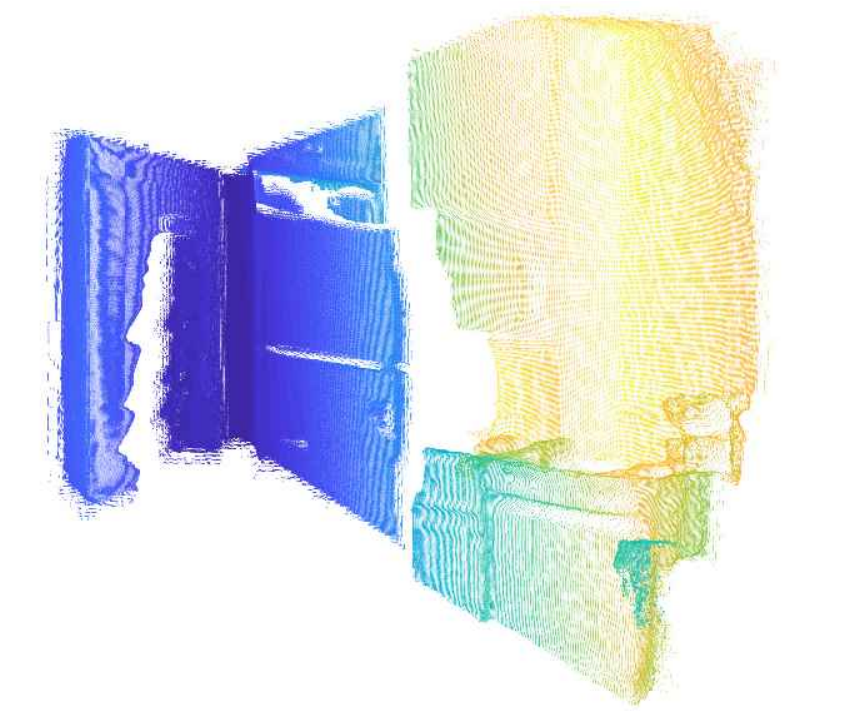} & 
\includegraphics[width=0.22\textwidth]{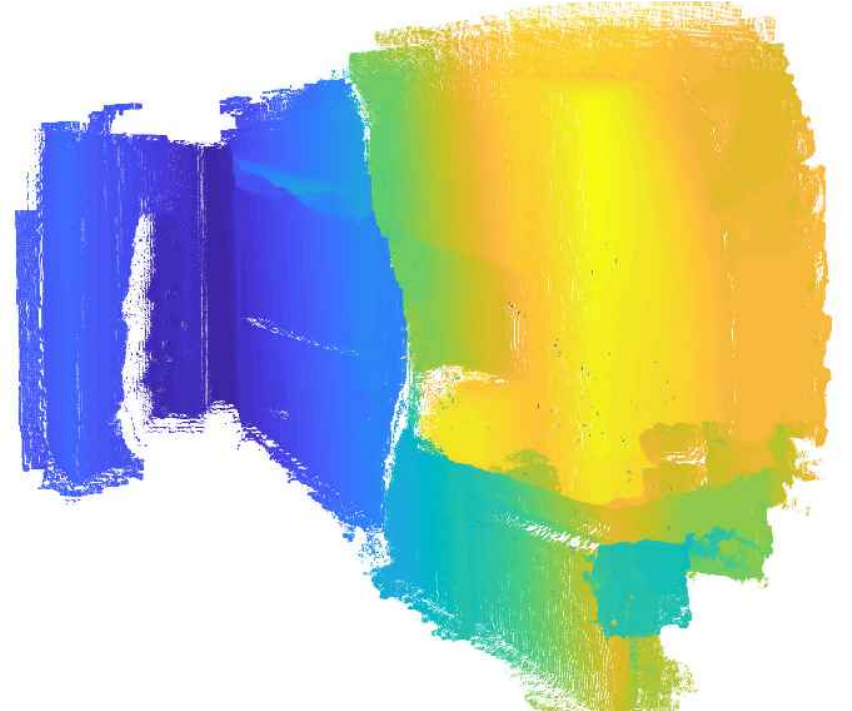} & 
\includegraphics[width=0.22\textwidth]{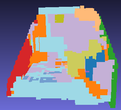} \\

\includegraphics[width=0.22\textwidth]{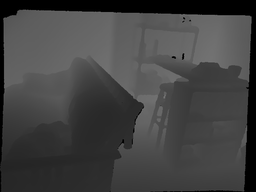} & 
\includegraphics[width=0.22\textwidth]{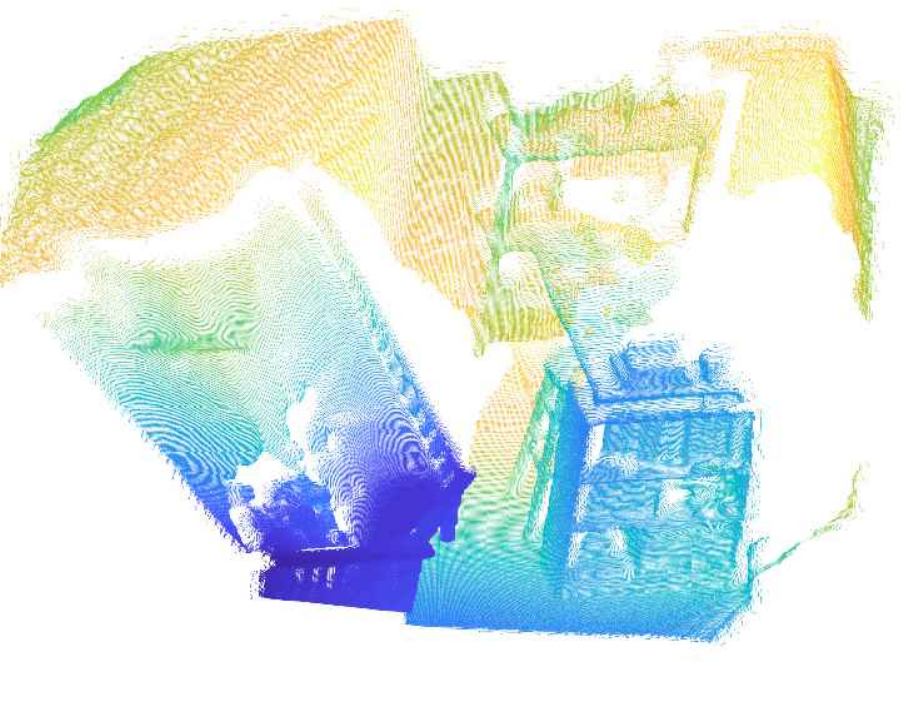} & 
\includegraphics[width=0.22\textwidth]{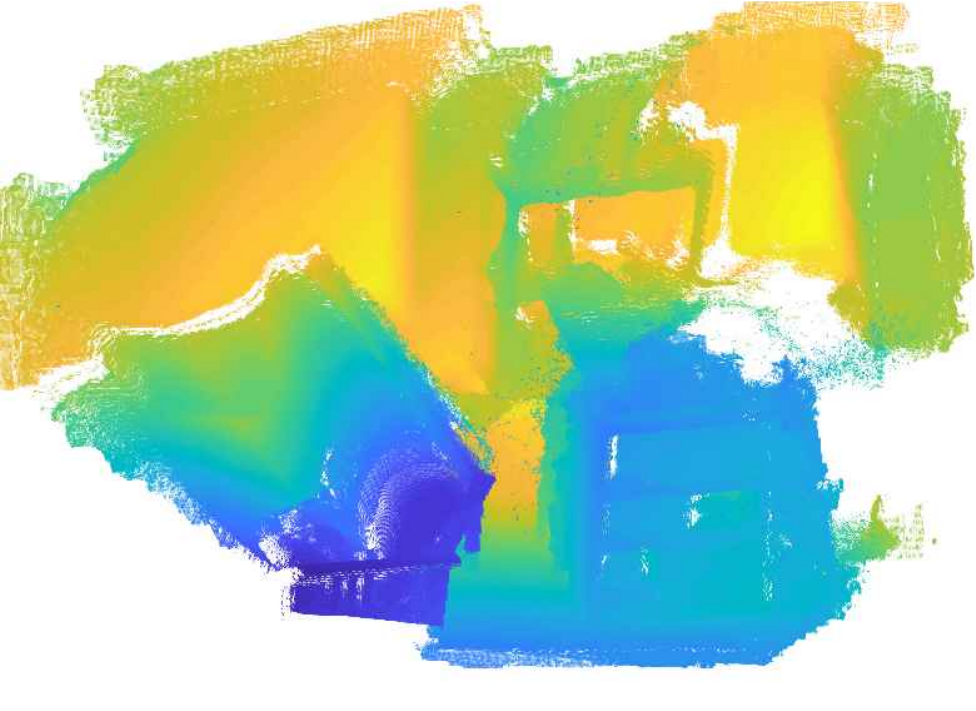} & 
\includegraphics[width=0.22\textwidth]{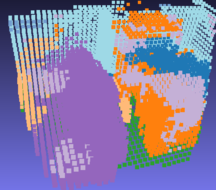} \\

\includegraphics[width=0.22\textwidth]{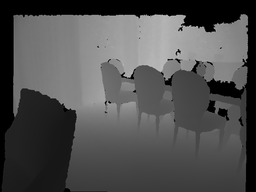} & 
\includegraphics[width=0.22\textwidth]{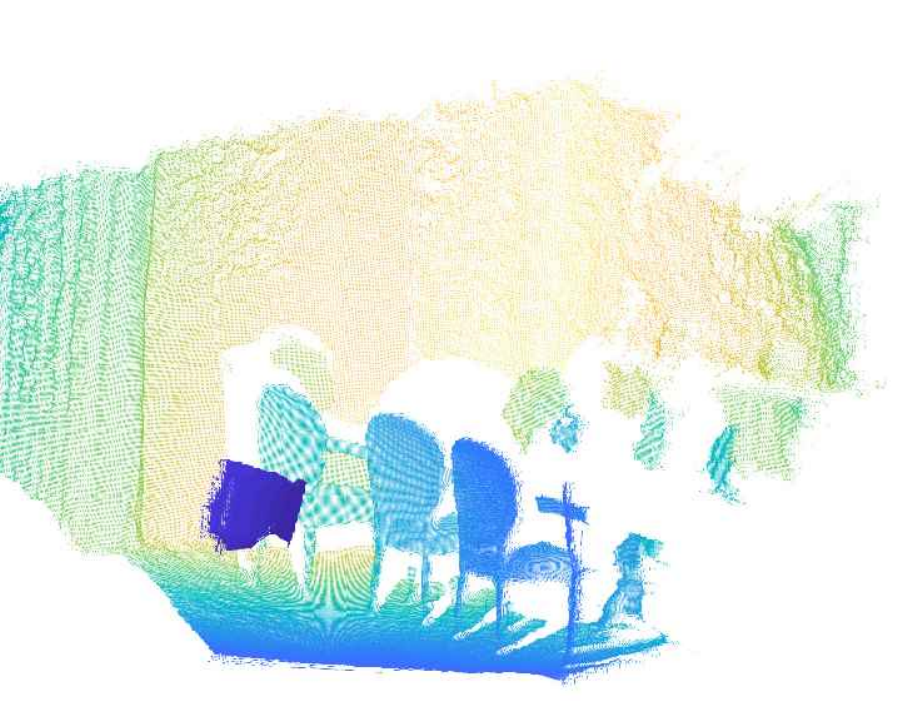} & 
\includegraphics[width=0.22\textwidth]{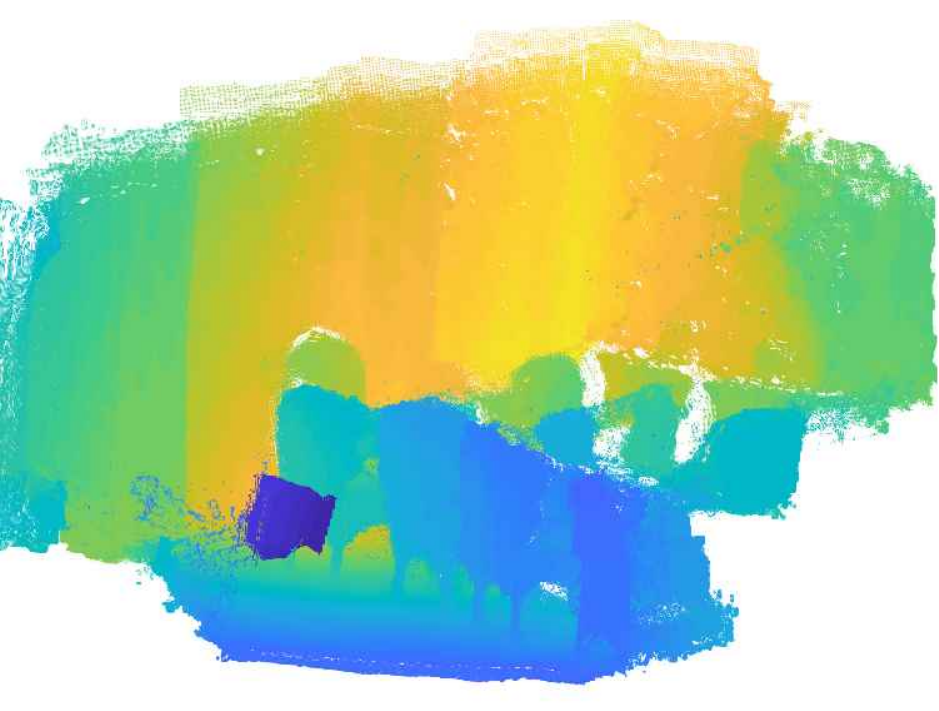} & 
\includegraphics[width=0.22\textwidth]{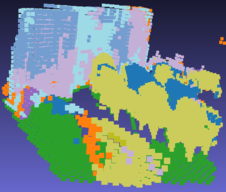} \\

\includegraphics[width=0.22\textwidth]{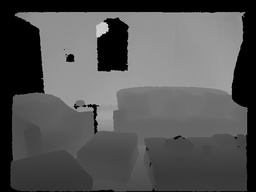} & 
\includegraphics[width=0.22\textwidth]{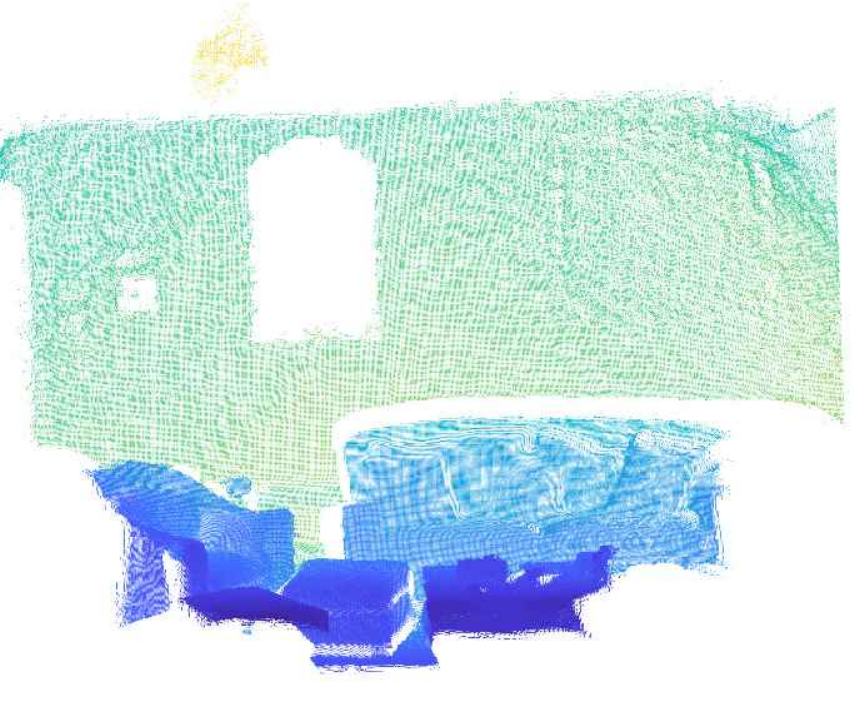} & 
\includegraphics[width=0.22\textwidth]{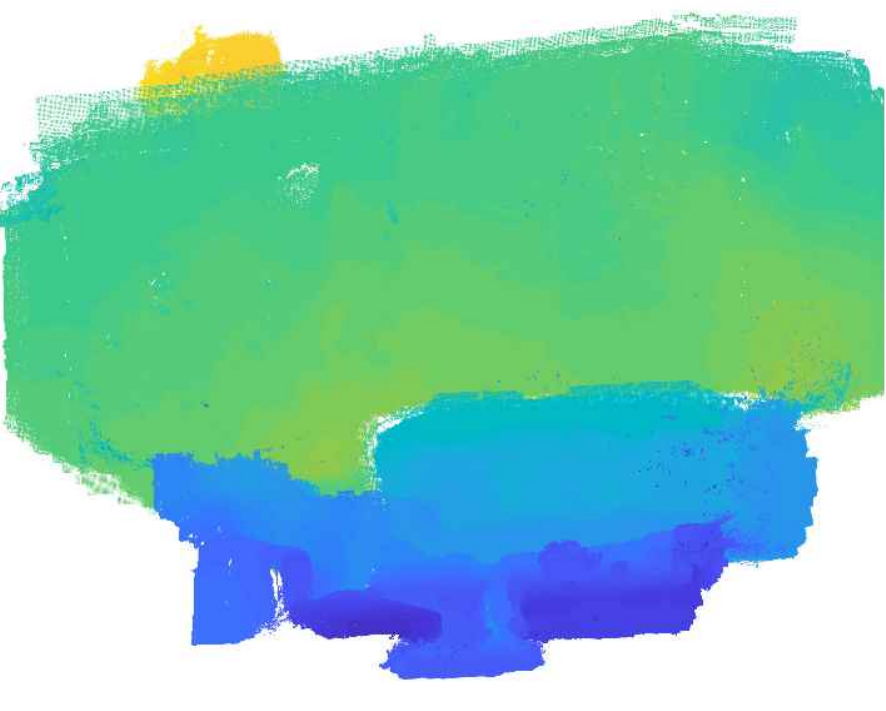} & 
\includegraphics[width=0.22\textwidth]{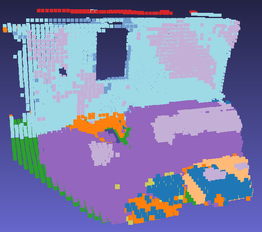} \\

\includegraphics[width=0.22\textwidth]{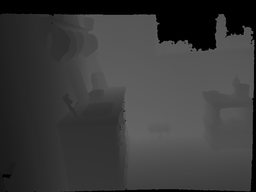} & 
\includegraphics[width=0.22\textwidth]{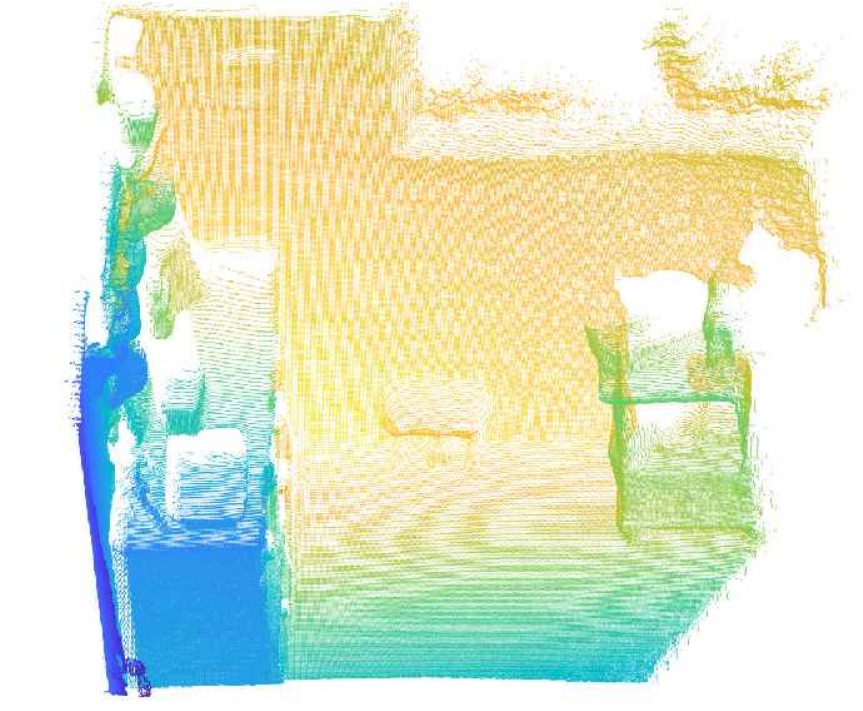} & 
\includegraphics[width=0.22\textwidth]{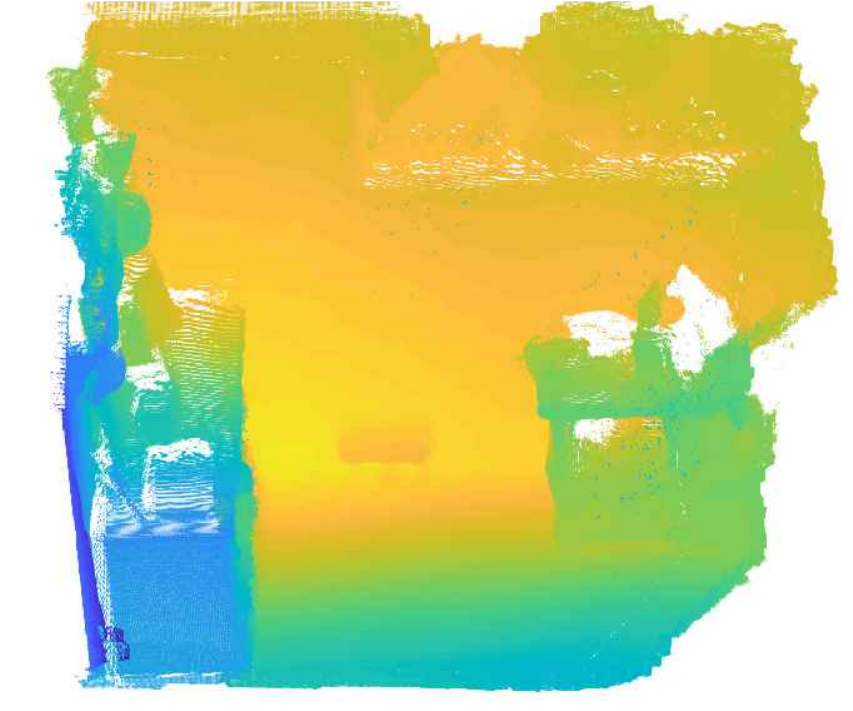} & 
\includegraphics[width=0.22\textwidth]{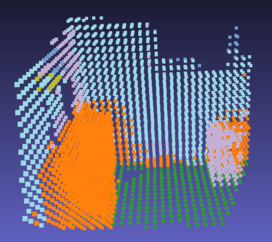} \\

(a) Input Depth & (b) Projected in 3D & (c) Proposed completion & (d) Completion by SSC-Net~\cite{song2017semantic} 
\end{tabular}
\caption{(a) Qualitative results of depth completion methods at different viewing location and orientation on NYU~\cite{Silberman:ECCV12} datasets. Images are first warped to the target pose and then use proposed depth completion method to predict depth at the occluded regions and subsequently merged.}
\label{fig:results10}
\end{figure*}

\begin{figure*}
\centering \scriptsize  
\begin{tabular}{c@{\hspace{0.05cm}}c@{\hspace{0.05cm}}c@{\hspace{0.05cm}}c}

\includegraphics[width=0.22\textwidth]{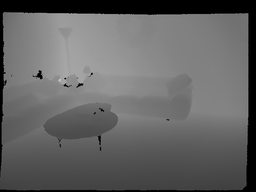} & 
\includegraphics[width=0.22\textwidth]{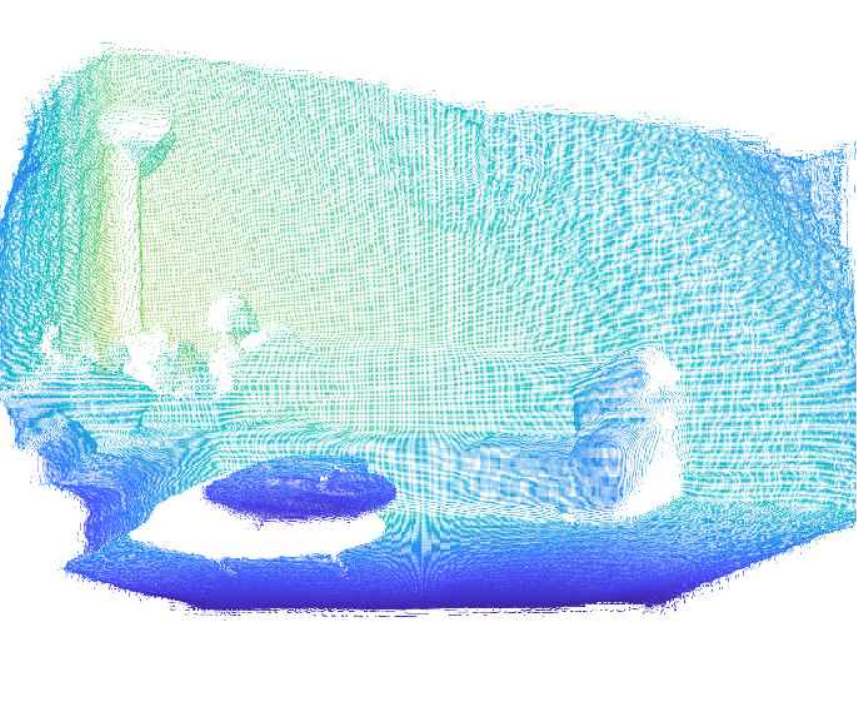} & 
\includegraphics[width=0.22\textwidth]{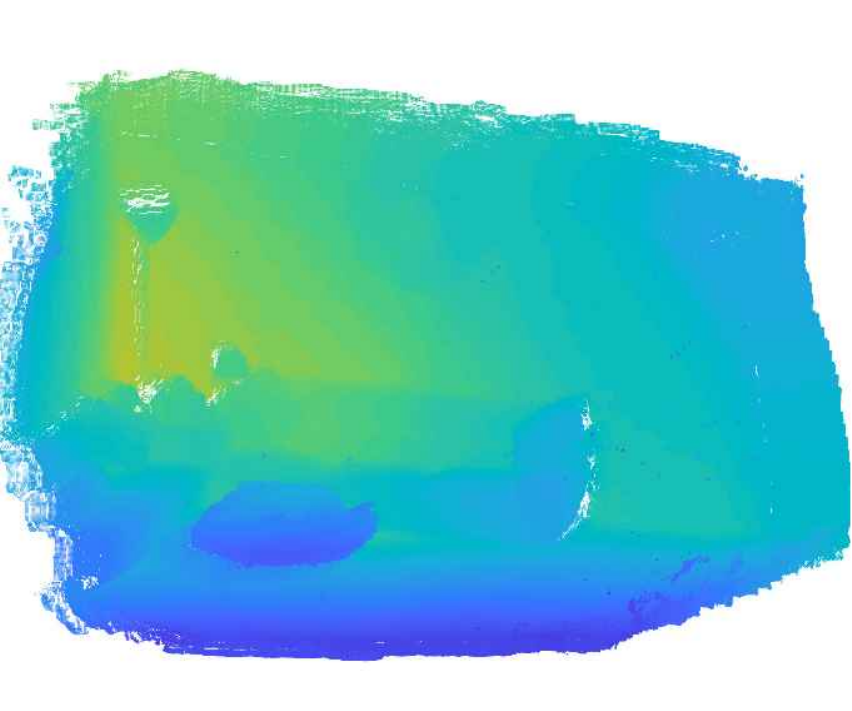} & 
\includegraphics[width=0.22\textwidth]{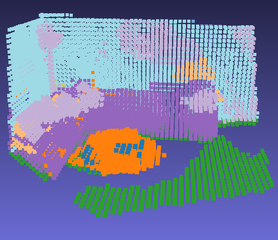} \\

\includegraphics[width=0.22\textwidth]{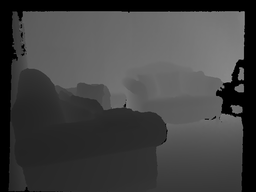} & 
\includegraphics[width=0.22\textwidth]{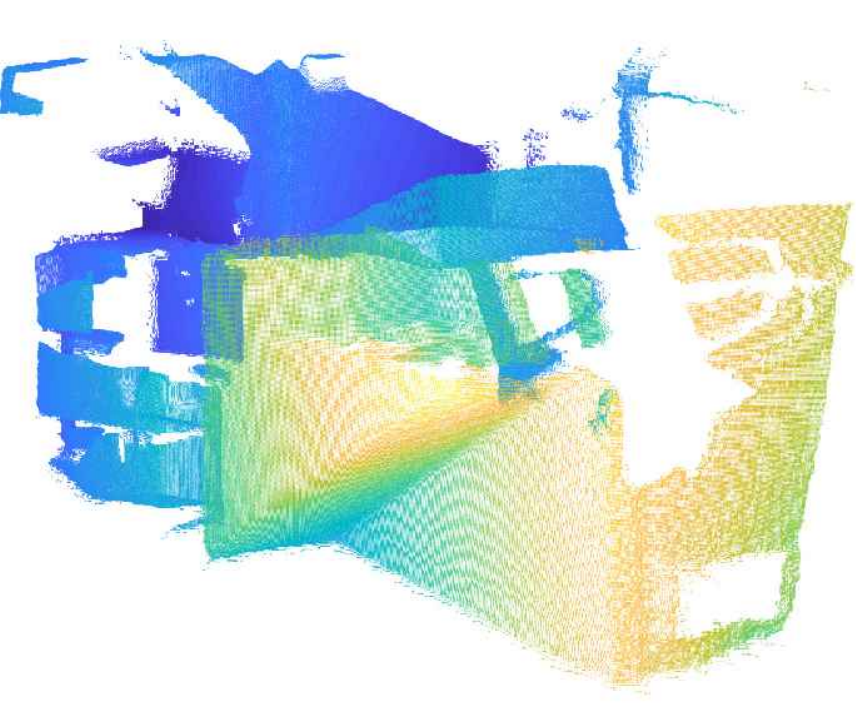} & 
\includegraphics[width=0.22\textwidth]{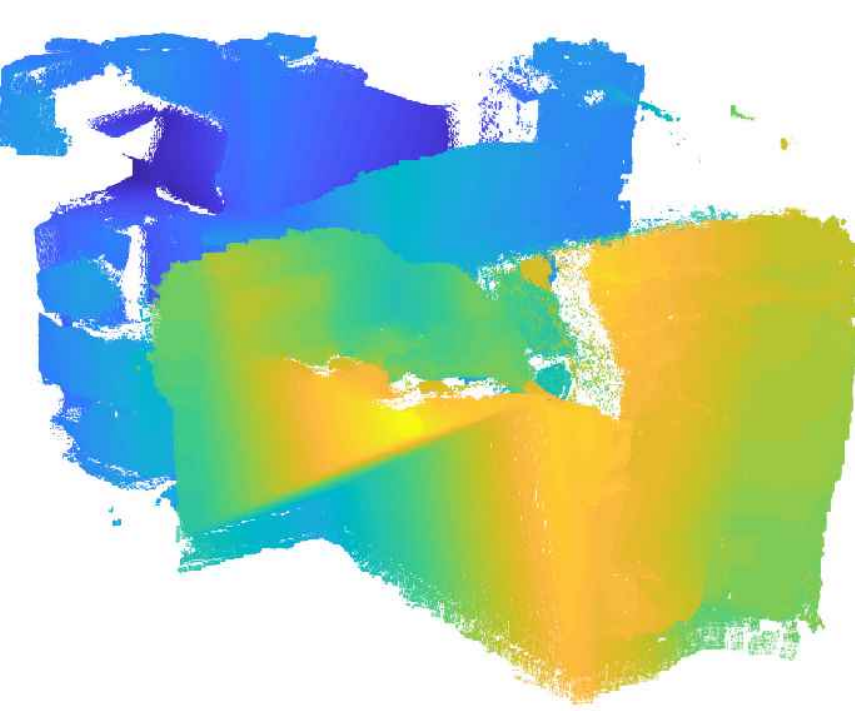} & 
\includegraphics[width=0.22\textwidth]{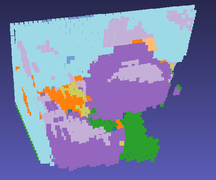} \\

\includegraphics[width=0.22\textwidth]{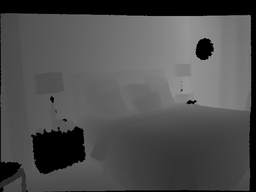} & 
\includegraphics[width=0.22\textwidth]{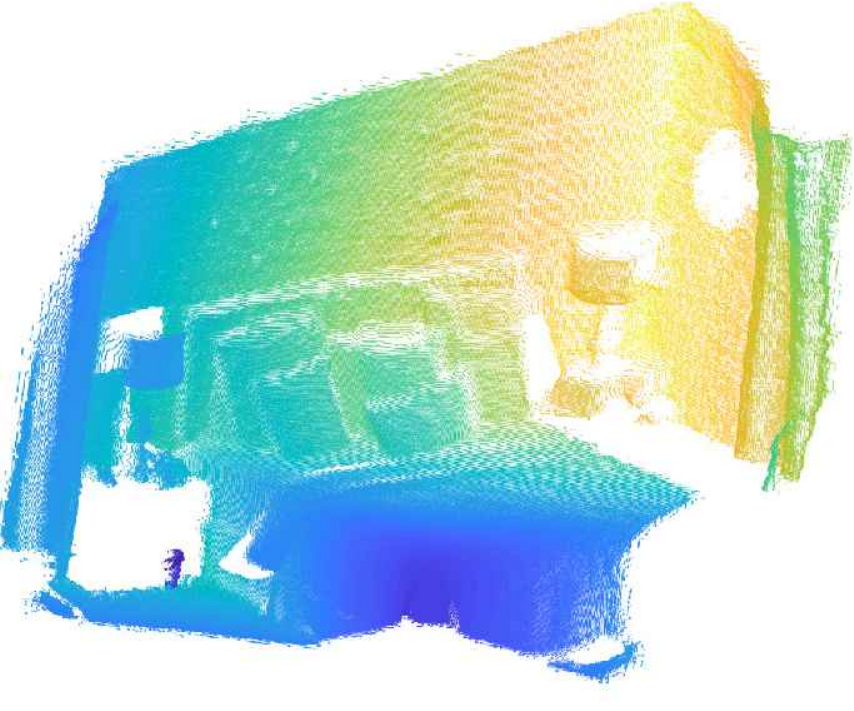} & 
\includegraphics[width=0.22\textwidth]{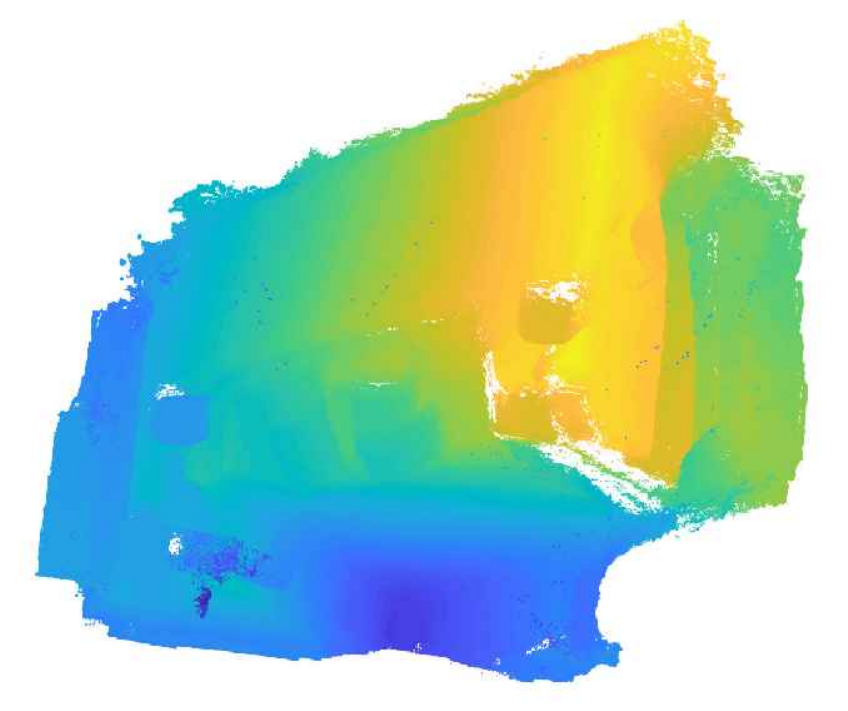} & 
\includegraphics[width=0.22\textwidth]{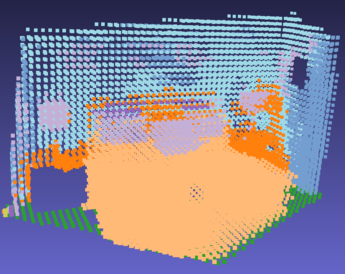} \\

\includegraphics[width=0.22\textwidth]{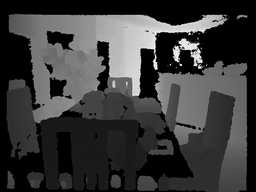} & 
\includegraphics[width=0.22\textwidth]{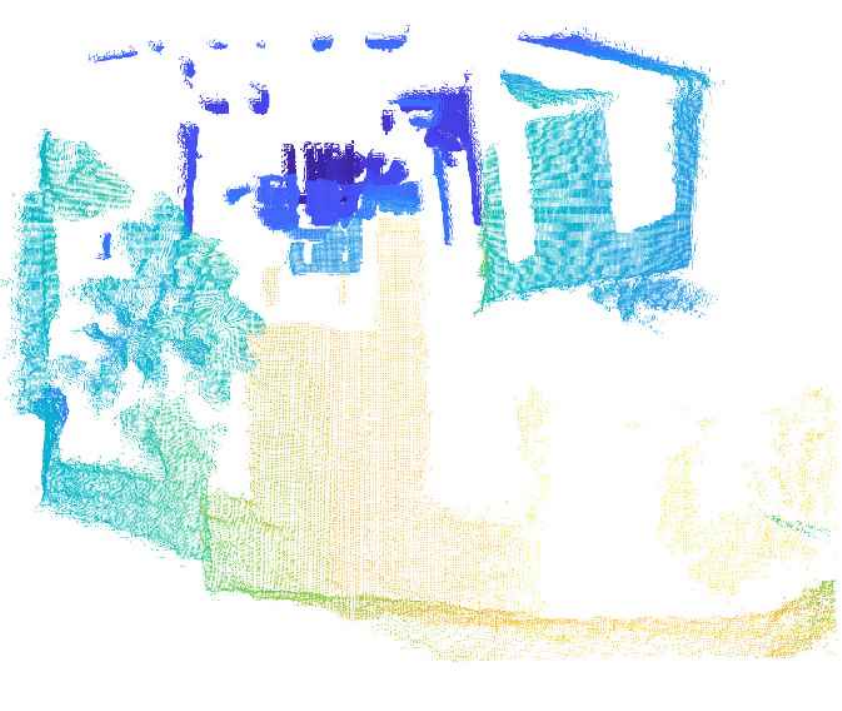} & 
\includegraphics[width=0.22\textwidth]{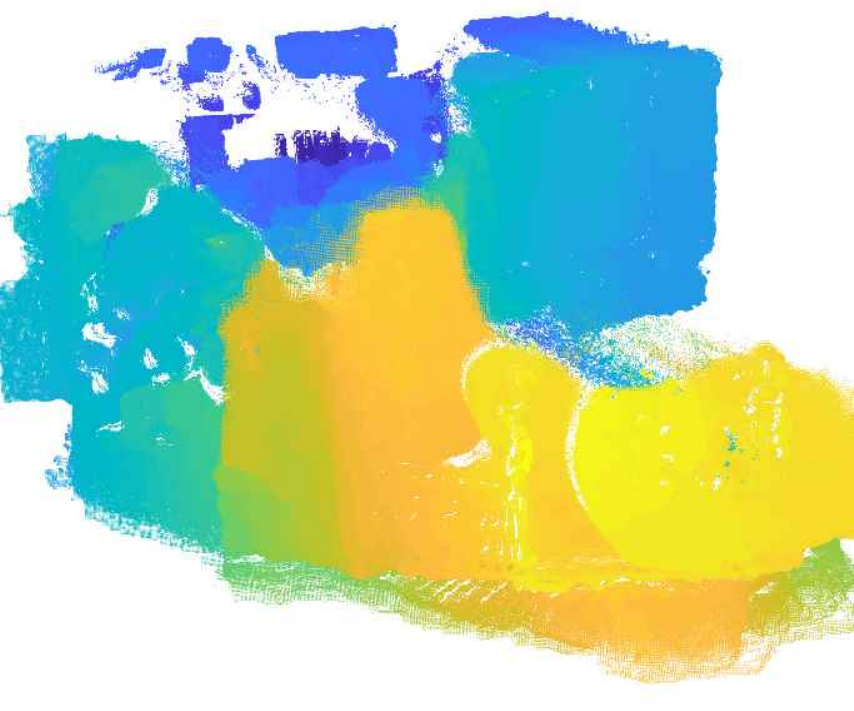} & 
\includegraphics[width=0.22\textwidth]{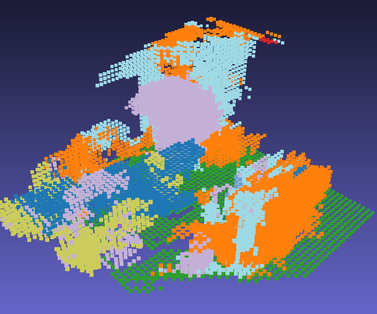} \\

\includegraphics[width=0.22\textwidth]{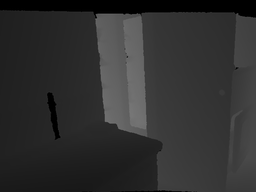} & 
\includegraphics[width=0.22\textwidth]{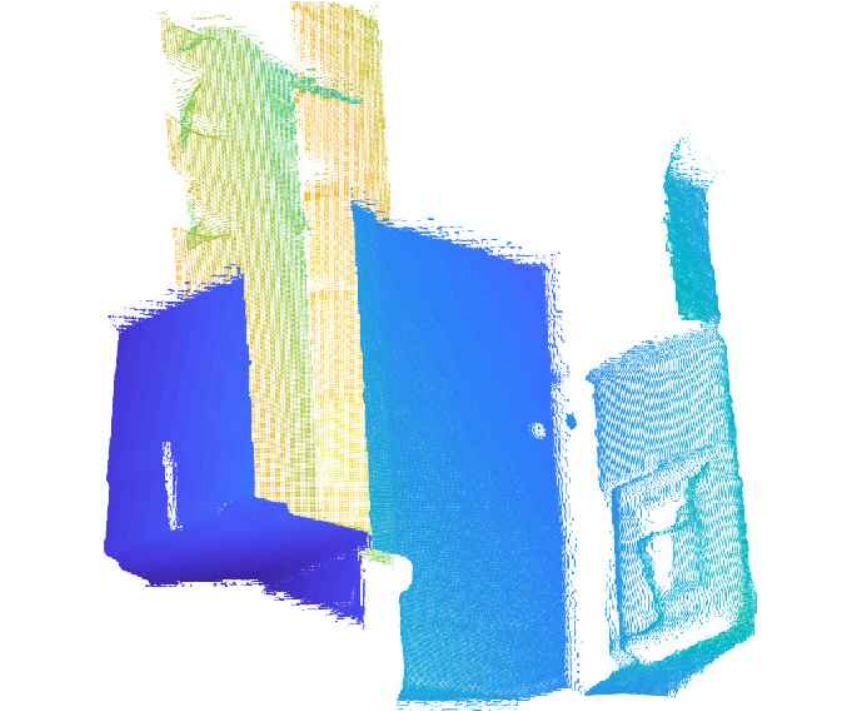} & 
\includegraphics[width=0.22\textwidth]{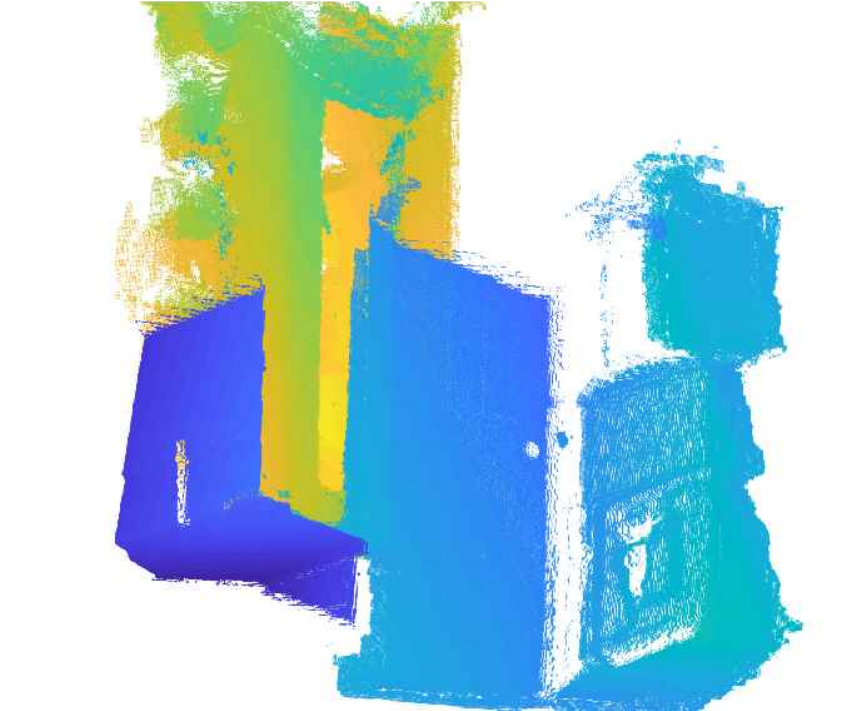} & 
\includegraphics[width=0.22\textwidth]{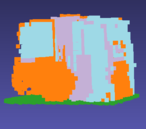} \\

\includegraphics[width=0.22\textwidth]{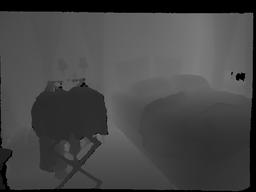} & 
\includegraphics[width=0.22\textwidth]{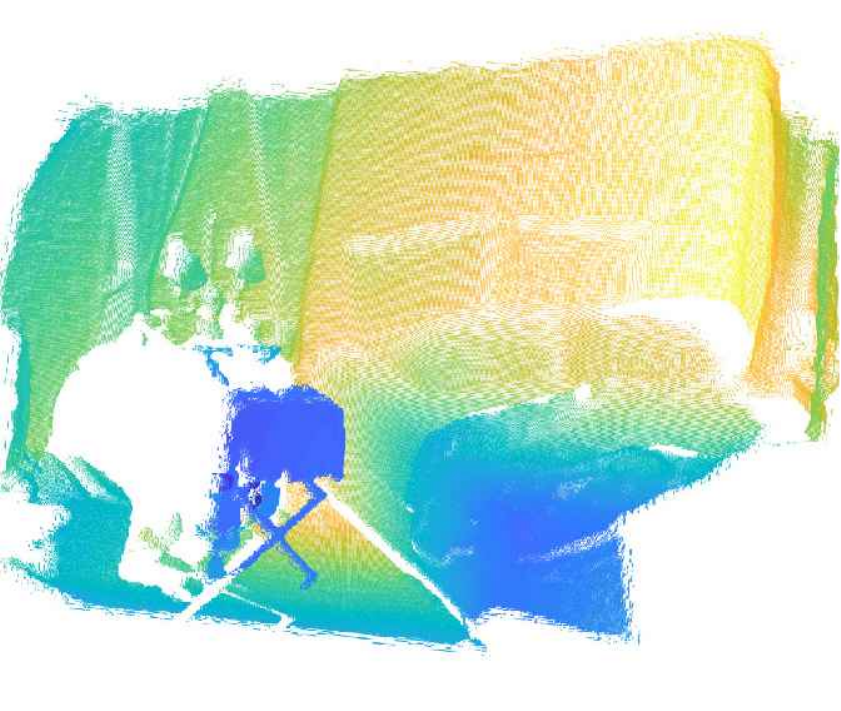} & 
\includegraphics[width=0.22\textwidth]{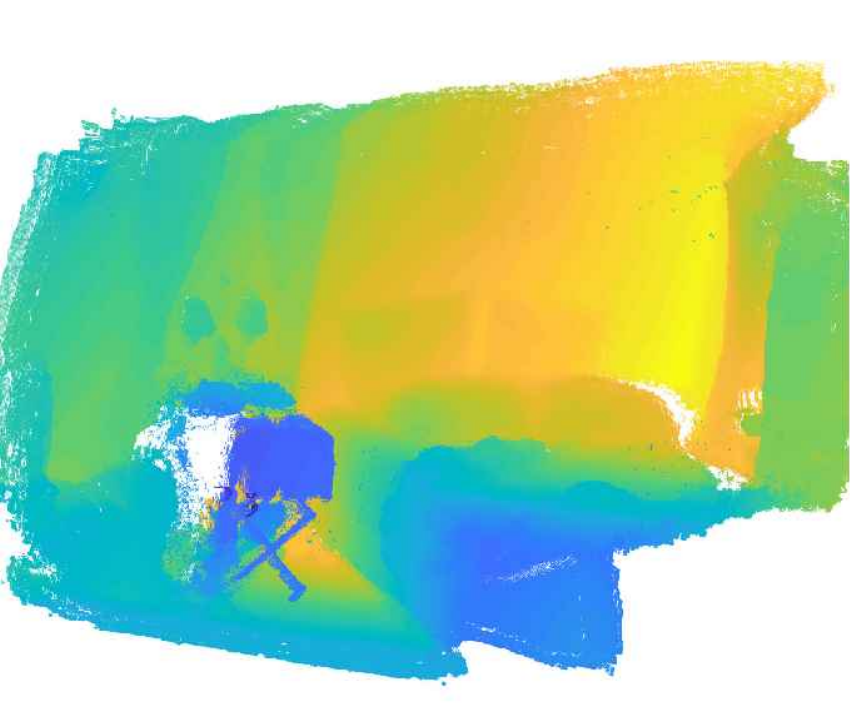} & 
\includegraphics[width=0.22\textwidth]{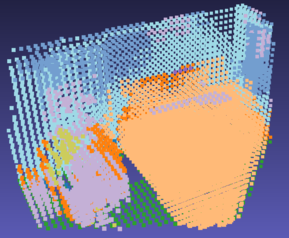} \\

(a) Input Depth & (b) Projected in 3D & (c) Proposed completion & (d) Completion by SSC-Net~\cite{song2017semantic} 
\end{tabular}
\caption{(a) Qualitative results of depth completion methods at different viewing location and orientation on NYU~\cite{Silberman:ECCV12} datasets. Images are first warped to the target pose and then use proposed depth completion method to predict depth at the occluded regions and subsequently merged.} 
\label{fig:results11}
\end{figure*}

\bibliographystyle{plainnat}
\bibliography{egbib}

\begin{thebibliography}{10}

\bibitem{dai2017scannet}
Angela Dai, Angel~X Chang, Manolis Savva, Maciej Halber, Thomas Funkhouser, and
  Matthias Niessner.
\newblock {ScanNet: Richly-Annotated 3D Reconstructions of Indoor Scenes}.
\newblock In {\em {Proc. of CVPR}}, pages 5828--5839, 2017.

\bibitem{dai2017scancomplete}
Angela Dai, Daniel Ritchie, Martin Bokeloh, Scott Reed, J{\"u}rgen Sturm, and
  Matthias Nie{\ss}ner.
\newblock {ScanComplete: Large-Scale Scene Completion and Semantic Segmentation
  for 3D Scans}.
\newblock {\em CVPR}, 2018.

\bibitem{eigen2014depth}
David Eigen, Christian Puhrsch, and Rob Fergus.
\newblock Depth map prediction from a single image using a multi-scale deep
  network.
\newblock In {\em Advances in neural information processing systems}, pages
  2366--2374, 2014.

\bibitem{esedoglu2002digital}
Selim Esedoglu and Jianhong Shen.
\newblock Digital inpainting based on the mumford--shah--euler image model.
\newblock {\em European Journal of Applied Mathematics}, 13(4):353--370, 2002.

\bibitem{gennari1989models}
John~H Gennari, Pat Langley, and Doug Fisher.
\newblock {Models of incremental concept formation}.
\newblock {\em Artificial intelligence}, 40(1-3):11--61, 1989.

\bibitem{guillemot2014image}
Christine Guillemot and Olivier {Le Meur}.
\newblock {Image inpainting: Overview and recent advances}.
\newblock {\em IEEE signal processing magazine}, 31(1):127--144, 2014.

\bibitem{iizuka2017globally}
Satoshi Iizuka, Edgar Simo-Serra, and Hiroshi Ishikawa.
\newblock {Globally and locally consistent image completion}.
\newblock {\em ACM Transactions on Graphics (TOG)}, 36(4):107, 2017.

\bibitem{johnson2016perceptual}
Justin Johnson, Alexandre Alahi, and Li~Fei-Fei.
\newblock Perceptual losses for real-time style transfer and super-resolution.
\newblock In {\em European Conference on Computer Vision}, pages 694--711.
  Springer, 2016.

\bibitem{lai2014unsupervised}
Kevin Lai, Liefeng Bo, and Dieter Fox.
\newblock Unsupervised feature learning for 3d scene labeling.
\newblock In {\em Robotics and Automation (ICRA), 2014 IEEE International
  Conference on}, pages 3050--3057. IEEE, 2014.

\bibitem{marr1982computational}
DC~Marr.
\newblock {A computational investigation into the human representation and
  processing of visual information}.
\newblock {\em Freeman, San Francisco, CA}, 1982.

\bibitem{merrell2007realtime}
Paul Merrell, Amir Akbarzadeh, Liang Wang, Philippos Mordohai, Jan-Michael
  Frahm, Ruigang Yang, David Nist{\'e}r, and Marc Pollefeys.
\newblock Real-time visibility-based fusion of depth maps.
\newblock In {\em IEEE 11th International Conference on Computer Vision
  (ICCV)}, pages 1--8. IEEE, 2007.

\bibitem{Silberman:ECCV12}
Derek Hoiem Pushmeet~Kohli {Nathan Silberman} and Rob Fergus.
\newblock {Indoor Segmentation and Support Inference from RGBD Images}.
\newblock In {\em {Proc. of ECCV}}, pages 746--760, 2012.

\bibitem{park2017transformation}
Eunbyung Park, Jimei Yang, Ersin Yumer, Duygu Ceylan, and Alexander~C Berg.
\newblock {Transformation-grounded image generation network for novel 3d view
  synthesis}.
\newblock In {\em {2017 IEEE Conference on Computer Vision and Pattern
  Recognition (CVPR)}}, pages 702--711. IEEE, 2017.

\bibitem{pertuz2017region}
Said Pertuz and Joni Kamarainen.
\newblock {Region-based depth recovery for highly sparse depth maps}.
\newblock In {\em {Image Processing (ICIP), 2017 IEEE International Conference
  on}}, pages 2074--2078. IEEE, 2017.

\bibitem{ronneberger2015u}
Olaf Ronneberger, Philipp Fischer, and Thomas Brox.
\newblock {U-net: Convolutional networks for biomedical image segmentation}.
\newblock In {\em {International Conference on Medical image computing and
  computer-assisted intervention}}, pages 234--241. Springer, 2015.

\bibitem{shao2013computer}
Ling Shao, Jungong Han, Dong Xu, and Jamie Shotton.
\newblock {Computer vision for RGB-D sensors: Kinect and its applications
  [special issue intro.]}.
\newblock {\em IEEE transactions on cybernetics}, 43(5):1314--1317, 2013.

\bibitem{shen2002mathematical}
Jianhong Shen and Tony~F Chan.
\newblock Mathematical models for local nontexture inpaintings.
\newblock {\em SIAM Journal on Applied Mathematics}, 62(3):1019--1043, 2002.

\bibitem{song2017semantic}
Shuran Song, Fisher Yu, Andy Zeng, Angel~X Chang, Manolis Savva, and Thomas
  Funkhouser.
\newblock {Semantic scene completion from a single depth image}.
\newblock In {\em {Proc. of CVPR}}, pages 190--198, 2017.

\bibitem{wu2017marrnet}
Jiajun Wu, Yifan Wang, Tianfan Xue, Xingyuan Sun, Bill Freeman, and Josh
  Tenenbaum.
\newblock {Marrnet: 3d shape reconstruction via 2.5 d sketches}.
\newblock In {\em {Advances In Neural Information Processing Systems}}, pages
  540--550, 2017.

\bibitem{xue2017depth}
Hongyang Xue, Shengming Zhang, and Deng Cai.
\newblock {Depth image inpainting: Improving low rank matrix completion with
  low gradient regularization}.
\newblock {\em IEEE Transactions on Image Processing}, 26(9):4311--4320, 2017.

\bibitem{yan2018shift}
Zhaoyi Yan, Xiaoming Li, Mu~Li, Wangmeng Zuo, and Shiguang Shan.
\newblock Shift-net: Image inpainting via deep feature rearrangement.
\newblock {\em arXiv preprint arXiv:1801.09392}, 2018.

\bibitem{zelek2017point}
John Zelek and Nolan Lunscher.
\newblock {Point cloud completion of foot shape from a single depth map for fit
  matching using deep learning view synthesis}.
\newblock In {\em {Computer Vision Workshop (ICCVW), 2017 IEEE International
  Conference on}}, pages 2300--2305. IEEE, 2017.

\bibitem{zhang2018deep}
Yinda Zhang and Thomas Funkhouser.
\newblock {Deep Depth Completion of a Single RGB-D Image}.
\newblock {\em arXiv preprint arXiv:1803.09326}, 2018.

\end{thebibliography}

\end{document}